\documentclass{article}

\PassOptionsToPackage{numbers, compress}{natbib}

\usepackage[final]{neurips_2025}

\usepackage[table]{xcolor}

\usepackage{amsmath, amsthm, amssymb}

\usepackage{microtype}
\usepackage{enumitem}
\usepackage{url}

\usepackage{graphicx}
\usepackage{algorithm}
\usepackage{algpseudocode} % 来自 algorithmicx 的伪代码环境
\usepackage{amsmath}

\usepackage{booktabs}
\usepackage{multirow}
\usepackage{subcaption}
\usepackage[colorlinks=true, linkcolor=brown!60!black, citecolor=brown!60!black, urlcolor=pink]{hyperref}
\usepackage{colortbl}

\usepackage[capitalise]{cleveref}

\usepackage{xcolor}
\usepackage{wrapfig}   % For wrapping figures inside text

\makeatletter
\newif\ifeq@done \eq@donefalse
\newif\ifcorr@done \corr@donefalse

\newcommand{\eqcontrib}{%
  \footnotemark[1]%
  \ifeq@done\relax\else
    \g@addto@macro\@thanks{\footnotetext[1]{Equal contribution.}}%
    \eq@donetrue
  \fi
}

\newcommand{\corrauthor}[1]{%
  \footnotemark[2]%
  \ifcorr@done\relax\else
    \g@addto@macro\@thanks{\footnotetext[2]{#1\vspace{-3.8ex}}}%
    \corr@donetrue
  \fi
}
\makeatother

\newtheorem{theorem}{Theorem}

\definecolor{bestcell}{RGB}{220,255,220}    % 绿色高亮最佳值
\definecolor{modelrow}{gray}{0.95}          % 模型组浅灰背景
\usepackage{colortbl}      % for \columncolor
\usepackage{array,xcolor}  % for custom column type

\usepackage{tcolorbox}

\tcbuselibrary{skins, breakable, listings, theorems} % Enable advanced styling

\definecolor{darkblue}{rgb}{0, 0, 0.5}



\title{\raisebox{-0.2cm}{\includegraphics[height=1cm]{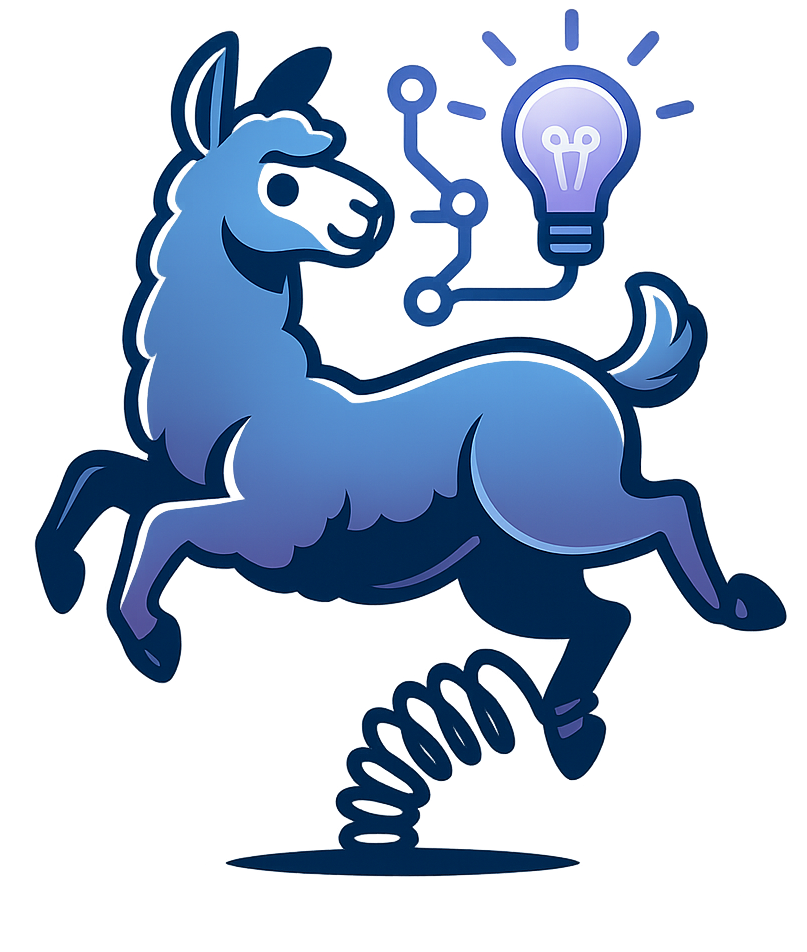}} AdaReasoner: Adaptive Reasoning Enables More Flexible Thinking}
 
\author{%
\begin{tabular}{c}
Xiangqi Wang\textsuperscript{1}\eqcontrib\quad
Yue Huang\textsuperscript{1}\eqcontrib\quad
Yanbo Wang\textsuperscript{2}\quad
Xiaonan Luo\textsuperscript{1}\quad
Kehan Guo\textsuperscript{1} \\
[-0.3ex] % 控制两行之间的紧凑程度
Yujun Zhou\textsuperscript{1}\quad
Xiangliang Zhang\textsuperscript{1}\corrauthor{Corresponding author: \texttt{xzhang33@nd.edu}}
\end{tabular}\\[3pt]
\textsuperscript{1} University of Notre Dame
\textsuperscript{2}MBZUAI\\[4pt]
\texttt{\{xwang76, yhuang37, xluo6, kguo2, yzhou25, xzhang33\}@nd.edu} \\
\texttt{yanbo.wang@mbzuai.ac.ae}
}

\begin{document}

\maketitle

\begin{abstract}
%Additionally, to counteract LLM inertia and its tendency to generate similar responses even when prompted differently, our approach generates multiple candidate responses, evaluates them through majority voting, and selects the best answer.

%Effective reasoning configuration is critical for large language models (LLMs) to for tasks that demand adaptive or creative insight—yet existing prompting methods rely on static setups and fail on challenges like joke generation or metaphor interpretation. 
LLMs often need effective configurations, like temperature and reasoning steps, to handle tasks requiring sophisticated reasoning and problem-solving, ranging from joke generation to mathematical reasoning. Existing prompting approaches usually adopt general-purpose, fixed configurations that work “well enough” across tasks but seldom achieve task-specific optimality.
To address this gap, we introduce AdaReasoner, an LLM-agnostic plugin designed for any LLM to automate adaptive reasoning configurations for tasks requiring different types of thinking. AdaReasoner is trained using a reinforcement learning (RL) framework, combining a factorized action space with a targeted exploration strategy, along with  a pretrained reward model to optimize the policy model for reasoning configurations with only a few-shot guide.
%a reinforcement learning–driven framework that dynamically optimizes reasoning configurations. AdaReasoner combines a factorized action space with a targeted exploration strategy and employs a pretrained reward model to assess response quality. 
AdaReasoner is backed by theoretical guarantees and experiments of fast convergence and a sublinear policy gap. Across six different LLMs and a variety of reasoning tasks, it consistently outperforms standard baselines, preserves out-of-distribution robustness, %being robust to reward noise, adapts to multiple RL algorithms, 
and yield gains on knowledge-intensive tasks through tailored prompts. Introduction of this paper can also be viewed publicly at \url{https://mine-lab-nd.github.io/project/adareasoner.html}.
%Chain-of-Thought (CoT) reasoning improves the problem-solving skills of large language models (LLMs), yet it often falls short on tasks that demand adaptive reasoning or creativity--such as crafting jokes or interpreting metaphors. To better tailor and optimize LLM adaptive reasoning capability for diverse question contexts, we introduce AdaReasoner, an adaptive framework based on few-shot learning that refines LLM reasoning configuration through reinforcement learning (RL). Specifically, AdaReasoner leverages reinforcement learning with few-shot examples to train an adapter that acts as an LLM reasoning guider, steering the LLM through distinct reasoning processes as guided by a pretrained reward model. For each question, it conducts multiple trials to produce a thorough evaluation, and with deliberate factorized action space makes estimation fast and accurate.  Experiments in mathematical reasoning, metaphor recognition, and various other reasoning tasks demonstrate that AdaReasoner significantly outperforms standard CoT methods and other baselines. Additional experiments prove that it's generalizable to Out-of-distribution (OOD) datasets, robust to reward signal, thorough in configuration design and compatible for other RL method.  Findings show that with specific instruction prompt configuration design, performance can also increase on knowledge-intensive questions.
\end{abstract}

\section{Introduction}

% \yue{I think you need to rewrite this introduction as it lacks logic and motivation. Here is a writing guidance for it: 1) LLM achievements + citations. 2) Reasoning prompt's achievements + improve reasoning capability. 3) For current related work, what's the disadvantage of them? 4) What are the challenges to solving these disadvantages? 5) Describe our method. 6) Contributions}

%Large Language Models (LLMs) have demonstrated remarkable progress in a variety of natural language processing tasks, from syntactic parsing to complex problem-solving \yue{Here needs citations}. A promising breakthrough in enhancing LLMs' reasoning abilities is the adoption of chain-of-thought (CoT) prompting, where the model generates intermediate logical steps before reaching the final answer \citep{wei2022chain}. \yue{Following that, there proposed many other variants of reasoning prompt techniques like ToT, xxx, and xxx (citations here).}

%\KG{The motivation section needs strengthening. To better justify the focus on prompting, clearly articulate its significant impact on reasoning task performance. Including specific examples demonstrating how different prompts lead to varying outcomes would greatly enhance this justification.}

Large Language Models (LLMs) have achieved impressive advancements across a wide range of natural language processing tasks, including syntactic parsing~\citep{ma2024llmparser}, complex scientific reasoning~\citep{wang2023scibench}, and commonsense knowledge answering~\citep{zhao2023large}. As the model size and training data scale up, LLMs have demonstrated the ability to surpass human-level accuracy on certain benchmarks~\citep{srivastava2022beyond}, highlighting their emerging capacity for sophisticated reasoning and problem-solving.

To better enhance LLM reasoning capabilities--and to push their performance closer to, or even beyond, human-level reasoning--numerous prompting-based strategies have been proposed. Chain-of-Thought (CoT) prompting encourages explicit decomposition of complex problems into intermediate steps~\citep{wei2022chain, zhou2022least}, while Tree-of-Thought (ToT) generalizes this idea by exploring multiple branching reasoning paths~\citep{yao2023tree}. Sampling-based approaches like Best-of-N improve robustness by selecting the most coherent reasoning path from diverse candidates~\citep{ji2023towards}, and automatic prompt optimization techniques aim to systematically discover prompts that better facilitate multi-step reasoning~\citep{zhang2022automatic, shum2023automatic}. If samples of the same type of question are provided, In-Context Learning (ICL)~\citep{brown2020language}   also prompts LLM with few-shot examples with advanced performance. 

Despite these advances, LLM reasoning remains highly configuration‐sensitive: as \autoref{fig:motivation} shows, GPT‑4o’s accuracy on the metaphor expression classification task~\citep{tong2024metaphor} swings wildly under different reasoning configurations. While divergent reasoning prompts and fewer reasoning steps could greatly improve performance, temperature as 1 instead drown out useful reasoning with noise, negating any benefit from the added randomness. However, previous methods have not targeted tuning on these parameters. CoT~\citep{wei2022chain, zhou2022least} and ToT~\citep{yao2023tree} apply fixed reasoning structures that fail to generalize to creative or subjective tasks (e.g. spatial planning \citep{stechly2024chain}). Best-of-N~\citep{ji2023towards} rely on unguided generation,  suffering from a ``garbage in, garbage out'' effect. Automatic prompt optimization~\citep{zhang2022automatic, shum2023automatic}  focuses on static templates and overlooks crucial hyperparameters like temperature, failing  to adjust  reasoning strategies.  While ICL~\citep{brown2020language}   extracts some cues from input questions, it remains brittle under context perturbations~\citep{mueller2023context}, and its reliance on implicit pattern matching has been shown to be less effective than direct structured reasoning~\citep{tang2023large}. These limitations call a need for an adaptive prompting configuration strategy for LLMs to handle various sophisticated reasoning. 

%However, overcoming these limitations entails several core challenges. First, the absence of a unified, adaptive reasoning mechanism poses a fundamental risk to the general usability of LLMs across tasks with diverse reasoning demands. Current strategies are brittle and overly specialized; without an adaptive framework that generalizes across logical, creative, and subjective domains, LLMs will remain fragile tools--impressive in benchmarks but unreliable in real-world deployments where task formats are fluid and unpredictable. Second, the demand for robust generalization often comes at the cost of massive computational overhead, a bottleneck that significantly limits practical deployment. If adaptive reasoning strategies cannot be made both effective and efficient, we risk creating methods that work only in idealized academic settings, failing to transfer to scalable, on-the-fly applications.

\begin{wrapfigure}{r}{0.35\textwidth}
    \centering
    \small
    \vspace{-10pt} 
    \includegraphics[width=\linewidth]{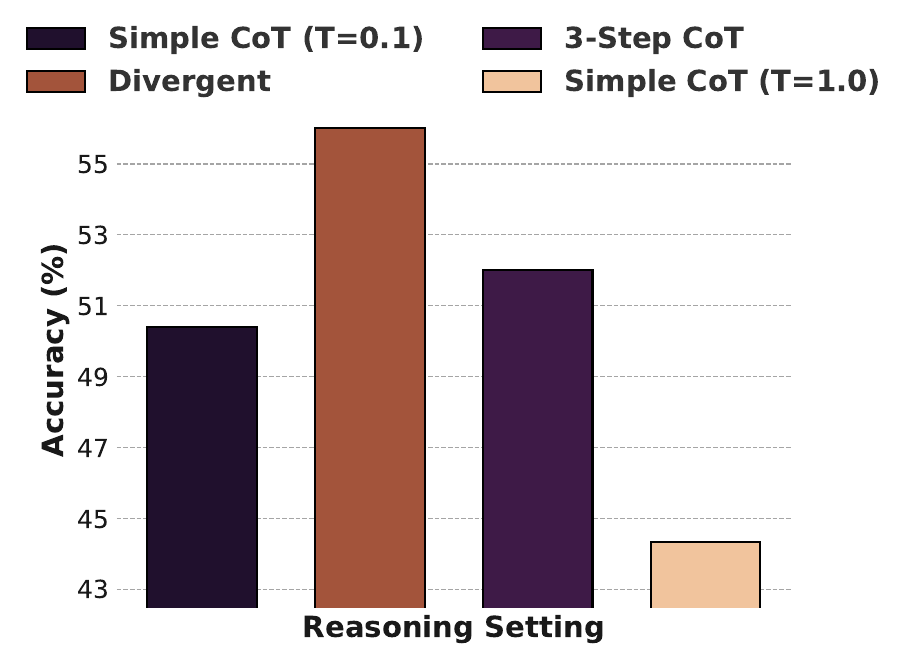} 
    \vspace{-12pt}
    \caption{\small{Performance of different CoT settings on the metaphor dataset~\citep{tong2024metaphor}. 
    %\textbf{3-Step}: Limits reasoning to 3 steps. 
    %\textbf{Divergent}: Encourages diverse reasoning. 
    %\textbf{Temp.=1}: Uses a simple CoT prompt with temperature=1. 
    The default temperature  is 0.1 if not specified.}}
    \vspace{-15pt} 
    \label{fig:motivation}
\end{wrapfigure}
%\textcolor{red}{Double check needed here.} 
However, identification of the optimal prompting configuration for LLMs is a non-trivial task. First, task types span logical, creative, and subjective domains, often in combination, so that many queries cannot be neatly categorized or matched with  pre-set configurations template. This necessitates strategies that are highly adaptive and tailored to the specific demands of each question. Second, LLM reasoning capability is sensitive to the configuration settings that involve multiple factors, as shown in \autoref{fig:motivation}.  The search space spanned by these factors when selecting effective configurations is combinatorially large. This presents a challenge for building a decision-making model that tailors the configuration for each input task.
Third, while building such a model using a data-driven approach is promising, exhaustively collecting training samples for every possible configuration is computationally expensive and impractical.
%shows that the multiple factors impose great change to LLM reasoning capability. Yet the configuration space—spanning temperature, reasoning steps, and reasoning method prompt choices—is combinatorially large, making it difficult to devise an optimal strategy in the configuration space for selecting effective configurations, especially in real-world deployments. Third exhaustively traversing all possible task samples is infeasible to build the strategy. 
This necessitates an approach that can generalize from a limited set of examples and capture reasoning patterns that are transferable across similar tasks.

We introduce AdaReasoner, an LLM-agnostic plugin designed to automate adaptive reasoning configurations for tasks requiring diverse types of thinking. When integrated with an LLM, AdaReasoner is trained using a reinforcement learning (RL) framework. In this setup, AdaReasoner acts as a decision-making agent, where the state is defined by the current task presented to the LLM, reflecting the nature of the reasoning required (e.g., logical, creative, or subjective). The action corresponds to selecting a configuration from an action space composed of three key hyperparameters: (i) the reasoning instruction formats, (ii) the generation temperature, and (iii) the number of reasoning steps. To enable AdaReasoner to learn the most effective configuration policy, a pretrained reward model is employed to evaluate the effectiveness of the reasoning configuration. This model provides feedback to guide the agent's learning,  enabling it to efficiently acquire effective configurations with only limited guidance (i.e., few-shot learning). To facilitate exploration and improve generalization, we employ a Boltzmann exploration mechanism, enabling the agent to explore and optimize configurations more effectively during training.
Once trained, AdaReasoner is used as a plug-in to the LLM, providing   adaptive reasoning configurations that allow the model to adjust its reasoning approach based on the task at hand.

Our contributions can be summarized as the followings:
\vspace{-0.1in}
\begin{itemize}[leftmargin=*, itemsep=0pt, parsep=0pt]
  \item We introduce AdaReasoner, an LLM-agnostic plugin that automates adaptive reasoning configurations for tasks requiring diverse types of thinking. %the first adaptive reasoning framework for LLMs, leveraging factorized few-shot reinforcement learning to guide diverse reasoning paths.
  \item AdaReasoner leverages a reinforcement learning framework with a factorized action space.  %, enabling dynamic adjustment of reasoning strategies based on the task. 
  Its training is data-efficient yet scalable, requiring only a small number of samples for each task aided by the use of the Boltzmann exploration mechanism.
  \item Extensive evaluations on diverse tasks show that AdaReasoner outperforms standard CoT and baselines, and sustains strong OOD performance. % remains robust to reward-signal noise, compatible with multiple RL algorithms and thorough on design of action space.
  %\item Ablation studies confirm the contribution of each component and demonstrate that specific tailored prompt configurations can boost performance on knowledge-intensive tasks.
\end{itemize}

%\begin{itemize}[leftmargin=*, itemsep=0pt, parsep=0pt]
%  \item We propose the first adaptive reasoning framework for LLMs that incorporates self-reflection to guide reasoning and increase flexibility in generative model.
%  \item Extensive experimental evaluations reveal that our approach significantly outperforms conventional methods.
  %\item Ablation studies confirm the effectiveness of dynamically adapting each component of our reasoning toolbox, and suitability across all RL algorithms.
%\end{itemize}

%Extensive experiments on a range of inference tasks--including logical and commonsense reasoning, metaphor recognition, and trustful generation--demonstrate that our adaptive framework significantly outperforms traditional CoT prompting and self-consistency methods. Ablation studies show that xxx. Case studies also reveal that reasoning adjusting can't excel greatly at tasks challenging LLM's knowledge boundaries, which guides the next step of CoT refining method.

\vspace{-0.1in}
\section{Related Work of Reasoning in LLMs} \vspace{-0.08in}

The pursuit of enhanced reasoning capabilities in LLMs has spurred diverse research trajectories, beginning with foundational techniques like Chain-of-Thought (CoT) prompting \citep{wei2022chain}. CoT enables LLMs to articulate intermediate steps, significantly improving performance on complex tasks. However, its efficacy can be hampered by sensitivity to prompt formulation \citep{sprague2024cot, puerto2024fine} and limitations in subjective or creative domains \citep{chochlakis2024larger, xu2024exploring}, sometimes even degrading performance where brevity is key \citep{liu2024mind}. To mitigate these issues and reduce manual effort, innovations such as Automatic CoT (Auto-CoT) \citep{zhang2022automatic, shum2023automatic} emerged, automating the generation of effective reasoning exemplars. Further advancements include structured reasoning frameworks like Tree-of-Thoughts (ToT) \citep{yao2023tree} and Graph-of-Thoughts (GoT) \citep{besta2024graph}, which allow models to explore and evaluate multiple reasoning pathways, alongside methods like CoT-influx \citep{huang2023fewer} that optimize few-shot CoT contexts.

To bolster the robustness and reliability of LLM reasoning, researchers have explored self-correction and learning-based paradigms. Self-consistency techniques \citep{wang2022self}, often realized through Best-of-N sampling, leverage the generation of multiple diverse reasoning paths and subsequent aggregation (e.g., via majority voting) to improve answer accuracy. Complementary to this, self-reflection mechanisms, as seen in Self-Refine \citep{madaan2023self} and Reflexion \citep{shinn2023reflexion}, empower LLMs to iteratively critique and enhance their own outputs, akin to human error correction, with some approaches fine-tuning with divergent CoT to specifically boost these capabilities \citep{puerto2024fine}. Reinforcement Learning (RL) has also become a cornerstone for optimizing reasoning, from general alignment via RLHF \citep{ouyang2022training} to specialized reward models that guide the LLM towards more accurate and effective thought processes \citep{jaech2024openai}. Models like DeepSeek-R1 \citep{guo2025deepseek} exemplify LLMs fine-tuned with RL to excel at intricate reasoning, sometimes learning to control their own reasoning flow through meta-actions.

The nuanced control of generation parameters and adaptive hyperparameter tuning represent another critical frontier. The stochastic decoding settings, such as temperature, significantly affect output diversity and, consequently, reasoning quality and creativity \citep{renze-2024-effect}. Higher diversity can fuel methods like self-consistency but requires careful management to maintain coherence. Recent work has thus focused on automated optimization of prompt configurations, decoding parameters, and even enabling LLMs to self-regulate their generation strategies, as demonstrated by Hyperparameter-Aware Generation (HAG) \citep{wang2024llm}. Our AdaReasoner contributes to this line of research by introducing an adaptive framework that explicitly manages a toolbox of reasoning hyperparameters, including the reasoning method prompt, temperature, and number of reasoning steps, using an RL-trained agent to dynamically tailor the reasoning process to individual inputs, coupled with self-reflection and a robust selection mechanism for enhanced flexibility.

\section{AdaReasoner}

\textbf{Motivation.} Even though CoT and similar LLM reasoning methods have been studied as generally efficient and helpful,  they still cannot achieve ideal performance across all types of questions. For example, tasks like joke generation or metaphor interpretation often require divergent and creative reasoning chain ~\citep{zhong2024let}. For more complex reasoning tasks, stronger and more explicit reasoning instructions would be beneficial~\citep{lin2024critical}. Thus, adapting LLM configurations tailored for specific tasks is crucial for achieving better overall performance. As illustrated in~\autoref{fig:pipeline}, %we design AdaReasoner to adapt  reasoning configuration and employed parallel generation for robust performance.
we design AdaReasoner to adapt reasoning configurations by taking actions as a combination of different hyperparameters for LLMs. The inference/evaluation process is illustrated by the black arrows, while the training flow is depicted by the cyan arrows.

\begin{figure}[t]
    \centering
    \includegraphics[width=0.9\linewidth]{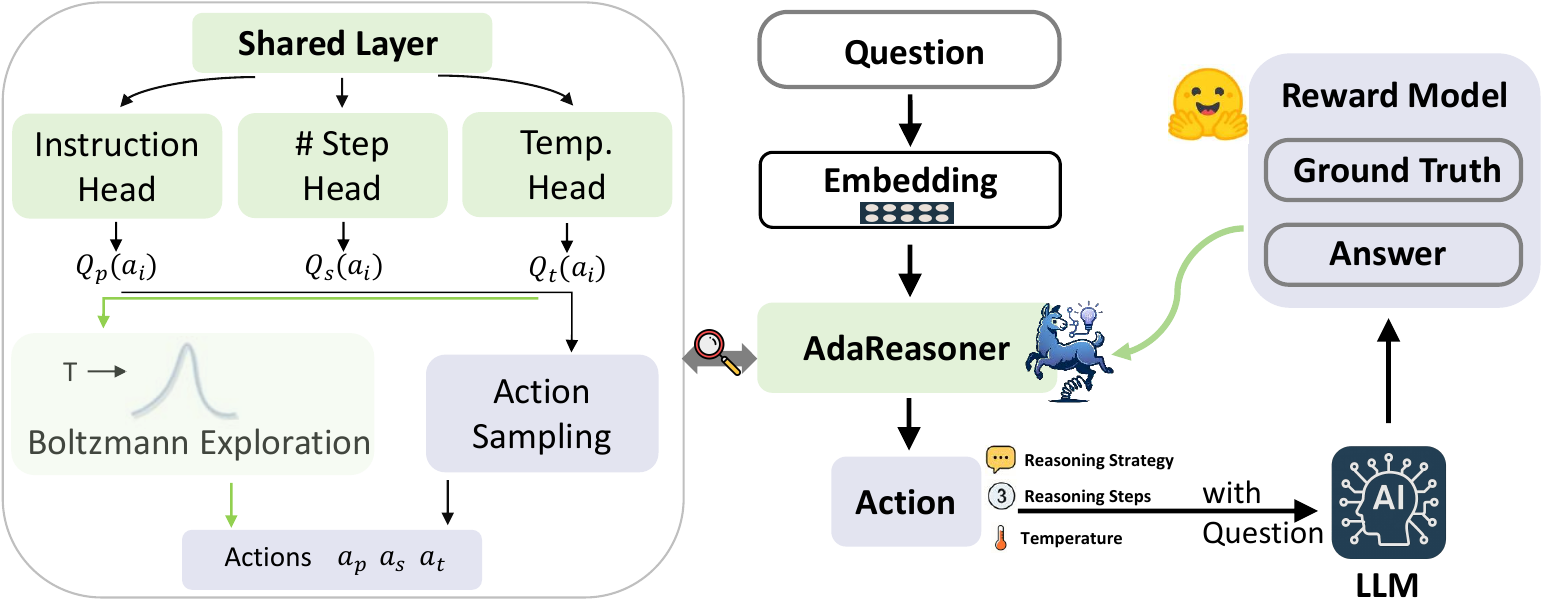}
    \caption{The proposed framework of using AdaReasoner for automating the % the AdaReasoner framework. Given a question embedding, AdaReasoner selects 
    reasoning configurations (instructions, steps, temperature). During training, configurations actions are sampled with Boltzmann exploration, guiding LLMs to generate answers, which are then evaluated by a reward model for policy optimization.}
    %\caption{This figure presents the AdaReasoner framework, comprising three stages: candidate answer generation via LLMs under varied hyperparameter configurations; adapter training using question embeddings and candidate answer rewards to optimize hyperparameter selection policy; and test-time inference, where the adapter guides LLMs to generate answers.}
    \label{fig:pipeline}
    \vspace{-10pt}
\end{figure}

\textbf{Problem Formulation.} The goal of AdaReasoner is to adaptively find the most effective hyperparameter configuration \( a \) for a given question \( q \) % given set or type question \( q \), 
such that an LLM (denoted as $\Phi$)  generates the correct reasoning answer $\Phi(q|a)$. More specifically,  the  configuration \( a \) is a 3-dimensional vector, where each element corresponds to one of the three hyperparameters: \(a_t\) (generation temperature), \(a_p\) (reasoning instruction format), and \(a_s\) (the number of reasoning steps). 
Denoting AdaReasoner as \(\Pi_{\Theta}\), our goal is to train its neural network weights \( \Theta\) to learn the optimal policy for deciding the configuration  \( a \) given a question \( q \). By considering the question \( q \) along with the LLM $\Phi$ as the state, the decision-making process is represented as \( a = \Pi_{\Theta} (q, \Phi)\).
During training, we employ a pre-trained model (e.g. DeBERTa in huggingface) as the reward model $r$ to provide feedback on the generated answer by comparing it to the ground truth reference $R$ from the training data, i.e., $r(\Phi(q|a),R)$.
In this approach, we address the issue that it is not possible to directly evaluate  the quality of generated configuration $a$,  as there is no ground truth for $a$ itself. Instead, the effectiveness of $a$  is judged indirectly based on the resulting answer  $\Phi(q|a)$, ensuring that the AdaReasoner agent is informed about the quality of its reasoning configuration through the answer's relevance and accuracy.

Within the broader RL framework, our study can be viewed as a \emph{multi-armed bandit} problem, where the \emph{arms} represent different configuration actions. Each question is an independent task (state), where the agent determines the actions (sets the values for all arms), receives a reward based on the effectiveness of the answer, and then moves on to the next task. The objective is to optimize the selection of hyperparameters to maximize the reward for each question.
Therefore,  given a set of training questions and reference answer samples  \(\mathcal{D}_{\text{train}} = \{(q_i, R_i)\}_{i=1}^M\), the objective is to train the AdaReasoner agent as 
\begin{equation}
\Theta^* = \arg\max_{\Theta} \ \mathbb{E}_{(q,R)\sim\mathcal{D}_{\text{train}}} \ \mathbb{E}_{a\sim\Pi_\Theta(a|q,\Phi)} \Big[ r\Big(\Phi(q|a),R\Big) \Big].
\label{eq:obj}
\end{equation}
Theoretical analysis about  AdaReasoner is presented in \cref{subsec: theory}, with a step-by-step description in Algorithm \ref{alg:adareasoner_per_question}.

\subsection{Hyperparameter Configuration (Action)}
\label{sec:ToolboxDesign}

%\yue{Missing details: 1) why do you select temperature, prompts, and reasoning methods? Give some citations to show their significance and support your motivation. 2) What's the difference between prompts and reasoning methods?}

As mentioned earlier, we consider three hyperparameters in the reasoning configuration: 1) the generation temperature \(a_t\); 2) the format of reasoning instructions  \(a_p\); and 3) the number of reasoning steps in CoT \(a_s\), for several reasons. First, they have substantial impacts on the reasoning performance. 
Previous studies have revealed that 
%To support both proprietary and open-weight LLMs, we introduce a comprehensive configuration schema encompassing generation temperature, reasoning method prompt and reasoning step instructions\footnote{We use temperature as the sole hyperparameter to guide the LLM, as Microsoft code source suggests that tuning top\_p alongside temperature is not recommended. \url{https://microsoft.github.io/autogen/0.2/docs/Use-Cases/enhanced_inference/}}. Previous studies have strengthened their effect in LLM generation. The 
the generation temperature modulates the diversity of model outputs, often yielding markedly different responses  when varied \citep{renze2024effect}. The number of reasoning steps reflects the depth and thoroughness of the inference process and it thus could influence the reasoning accuracy~\citep{dutta2024think, jin2024impact}.  The format of reasoning instructions, such as backward reasoning and step-by-step deduction, also plays a crucial role in guiding the model’s reasoning process \citep{almeida2024exploring, wang2025tutorial}.  % the prompting method harnesses diverse inferential strategies, such as backward reasoning, to enrich the model’s output \citep{almeida2024exploring, wang2025tutorial}.
Second, the settings of these three hyperparameters are adaptable for both proprietary and open-weight LLMs, with enhancement of adareasoner's versatility.
%. Focusing on these hyperparameters enhances our approach's versatility, making it applicable to any LLM.
Third, we are aware of other hyperparameters that may also impact reasoning, such as the $p$ in top-$p$ sampling during generation and the random seed. However, we exclude $p$ because tuning top-$p$ alongside temperature is not recommended  together with temperature \cite{autogen_enhanced_inference}.  Additionally, our empirical evaluation found that varying the random seed could not be beneficial for improving LLMs' reasoning performance (as shown in Section \ref{sec:ablation}).

%AdaReasoner is designed to be compatible with all RL algorithms and serves as a seamlessly transferable plug-in that utilizes LLM as a black-box generator.

%Hereby we detail our design of hyperparameter selection space. To ensure practicality, reasoning steps are constrained to avoid extreme values, while temperature is discretized with intervals into predefined values, both detailed in~\cref{app:tech} thus respectively constructing action space $\mathcal{A}_s, \mathcal{A}_t$ representing step number and temperature.

To ensure practical feasibility, these configuration actions are discretized with a finite set of options. Specifically, the number of reasoning steps is bounded to avoid extreme values, defined as $\mathcal{A}_s $ as integers set, and temperature is discretized as set $\mathcal{A}_t$. 
The options for reasoning instructions, denoted as $\mathcal{A}_p$, are constructed based on 
%We now elaborate on the design of our hyperparameter selection space. To ensure practical feasibility, the number of reasoning steps is bounded to avoid extreme values, and the temperature is discretized into predefined intervals. These design choices define the action spaces $\mathcal{A}_s$ and $\mathcal{A}_t$, corresponding to reasoning steps and temperature, respectively (see \cref{app:tech} for details). To construct the reasoning method prompt space, we adopt 
a compositional design grounded in structure-mapping theory from cognitive psychology~\citep{gentner1983structure}, which models human reasoning by composing a \textbf{core} reasoning structure with \textbf{contextual} modifications. Accordingly, each reasoning instruction is factorized into two components: a \textbf{base} component, which specifies the overall cognitive strategy (e.g., creative thinking, analogical mapping, self-audit~\citep{byrne2019human}), and a \textbf{variation}, which modulates the emphasis on specific parts of the question or modifies the reasoning surface form. 
For example, a \textbf{base} \emph{``Apply creative reasoning to unearth unconventional insights and challenge standard assumptions''} could be combined with a \textbf{variation} \emph{``Use simple, straightforward language to guarantee clarity and accessibility''} for guiding the reasoning of divergent thinking types of problems. The same  \textbf{base}, when combined with a \textbf{variation} \emph{Validate conclusions by aligning them with established principles or empirical data}, such instruction is useful for critical thinking types of reasoning problems.
%\textcolor{red}{?? here the sentence is not complete? Need another combination of base and variation to show the reasonability of combing them. }
%The same  \textbf{base}, when combined with a \textbf{variation} \emph{Validate conclusions by aligning them with established principles or empirical data}, such instruction is useful for critical thinking types of reasoning problems.
Detailed  of the base and variation components and their instantiation are provided in~\cref{app:tech}. 
The reasoning instruction action space, $\mathcal{A}_p$, is composed of pairs in the form of \{base, variation\}. Each action $a_p$ corresponds to one of the possible combinations of a base and its associated variation.
%This results in a factorized prompt action space, where $\mathcal{A}_p$ is formed as the product of the template and variation spaces. However, since both components address a unified reasoning objective, we treat $\mathcal{A}_p$ as a single integrated action space in our learning framework.

Ultimately, the decision about the action involves selecting a generation temperature $a_t$ from $\mathcal{A}_t$,  a number of reasoning steps $a_s$ from $\mathcal{A}_s$,  and one form of reasoning instruction $a_p$ from $\mathcal{A}_p$. 
\subsection{Design and Training of AdaReasoner}
\label{sec:ReinforcedReasoningGuide}

\textbf{Neural Architecture of AdaReasoner.}
As shown in~\autoref{fig:pipeline}, the input query question, after embedding, undergoes three action selections before being sent to the LLMs for reasoning to generate the answer. While the embedding is performed (e.g. by pre-trained BERT model~\citep{wolf2020transformers}), the trainable neural network parameters of AdaReason consist of three parallel channels, each corresponding to one action, and one shared common layer as in \autoref{fig:pipeline}. The workflow is as follows: let $Embed(q)$ be the embedding of the input question $q$. It is first passed through the common layer to obtain $h=f_{\theta_c}(Embed(q))$, where $\theta_c$ are the parameters of the common layer (e.g., a fully connected MLP), and $h$ captures the features necessary to determine the actions.   

Then $h$ is sent to each channel, where the action selection is performed as 
\begin{equation}
    a_p \sim \pi_p(\cdot | h) = f_{\theta_p}(h), \quad
    a_t \sim \pi_t(\cdot | h) = f_{\theta_t}(h), \quad
    a_s \sim \pi_s(\cdot | h) = f_{\theta_s}(h),
    \label{eq:9}
\end{equation}
where each policy $\pi(\cdot | h)$ is implemented as a feed-forward network. 

This design factorizes the policy $\Pi$ into three independent heads, each handling a specific action selection, significantly reducing optimization space from multiply to summary. Viewing $\Pi$ as multi-armed bandit problem, it is factorizing the joint-arms into set of parallel yet not independent single arm ones. While each head operates independently, they are optimized jointly with a shared latent representation, ensuring coherent decision-making and unified optimization across $a_p$, $a_s$ and $a_t$.
Let $K = M T$ be the total number of steps in learning, where $M$ is the number of training questions and $T$ is the number of trials for each question. We analyze the regret of AdaReasoner, i.e., the reward difference between AdaReasoner and the optimal policy without factorization in App.~\ref{app:reg}. The regret per step is bounded by $O\bigl(({\frac{|\mathcal A|\,\ln|\mathcal A|}{K}})^{0.5}\bigr)$, where $|\mathcal A|$ is the total number of action values:     $\mathcal{A} = \mathcal{A}_p \times \mathcal{A}_t \times \mathcal{A}_s$. This shows that the regret   per step becomes negligible once $K\gg|\mathcal A|\ln|\mathcal A|$, which is consistent with the empirical observation of few-shot convergence, meaning that AdaReasoner learns effectively with relatively few training examples.
%Setting $K = M T$ (with $M$ questions and $T$ trials each), AdaReasoner’s regret grows sublinearly as $O\bigl(\sqrt{K\,|\mathcal A|\,\ln|\mathcal A|}\bigr)$ (App.~\ref{app:reg}), so the per-question regret becomes negligible once $K\gg|\mathcal A|\ln|\mathcal A|$, aligning with empirical few-shot convergence. Furthermore, 
Moreover, under Lipschitz smoothness and bounded variance conditions, Adareasoner with $J(\Theta^*)$ denoted as optimal expected-reward objective and $J(\Theta_0)$ as initial objective achieves an error bound $\frac{2\bigl(J(\Theta^*)-J(\Theta_0)\bigr)}{\eta K} + L\;\eta\;\sigma^2$ (App.~\ref{app:fastcon}), reinforcing rapid convergence in the few-shot setting.

%As shown in \cref{app:converge}, the factorized Adareasoner has fast convergence with error rate bounded by $O\Bigl(\tfrac{1}{\alpha N} + \alpha\Bigr)$ \textcolor{red}{where $N$ is .... $\alpha$ is ....   When N.... the convergence is ..... Our empirical evaluation also .... This is a verification about the advantage of AdaReasoner in few-shot learning scenarios.  }.  Additionally,  the performance gap between the factorized policy and the non-factorized one is   bounded by    $\epsilon_{\mathrm{rep}} \;+\; 2\,\epsilon(N,\delta)$, \textcolor{red}{where $\epsilon_{\mathrm{rep}}$ is .... $\epsilon(N,\delta)$ is ....  }.

\textbf{Exploration Strategy.}
By formulating the configuration selection for each question as a multi-armed bandit (MAB) problem, we aim to design an effective exploration strategy under the few-shot training setting. However, since the reward is derived indirectly from LLM outputs and the process is not an online learning scenario, standard MAB strategies such as Upper Confidence Bound (UCB)~\citep{sutton2018reinforcement} become impractical. Moreover, evaluating all configurations for each context $q$ is computationally infeasible, especially given the noisy and implicit reward landscape induced by LLM responses. Therefore, it is crucial to explore broadly across the configuration space while still prioritizing high-reward actions, and
%As in solving a multi-armed bandit problem, our task requires balancing exploration and exploitation as the agent learns. Rather than exhaustively trying all possible configuration settings for each question (i.e., testing all possible values of the arms at one state), it is important to ensure broad coverage of the configuration space while avoiding premature convergence to suboptimal choices. 
Boltzmann exploration offers an effective solution~\citep{pan2019reinforcement}, as it allows the agent to probabilistically select actions based on their estimated rewards. Specifically, for each action  ($a_t$, $a_s$ or $a_p$), we estimate the selection probability for its all  possible values (in $\mathcal{A}_t$, $\mathcal{A}_s$ or $\mathcal{A}_p$),
\begin{equation}
P(a_i) = \frac{\exp\bigl(Q(a_i)/\tau\bigr)}{\sum_{a_j \in \mathcal{A}} \exp\bigl(Q(a_j)/\tau\bigr)},
\label{eq:2}
\end{equation}
where \(Q(a_i)\) is the logit score in the output layer of one policy network $f_\theta$ for  action \(a_i\). The temperature $\tau$ in Boltzmann exploration controls the exploration-exploitation trade-off: higher $\tau$ promotes exploration, lower $\tau$ favors exploitation. We anneal $\tau$ exponentially as $\tau_t = \tau_0 \cdot \alpha^t, t \le T$, allowing the policy to gradually shift from broad exploration to reliable configuration selection and refined optimization~\citep{kirkpatrick1983optimization}. 

%The temperature parameter $\tau$ here (for Boltzmann exploration) controls the balance between exploration and exploitation: higher temperatures encourage more exploration, while lower temperatures promote exploitation of the best-performing configurations. We set $\tau$ to be annealed exponentially, i.e., $\tau_k = \tau_0 \cdot \alpha^k$ at the $k$-th iteration. In this way,  policy transitioning from broad exploration to selecting learned effective configurations, ultimately providing reliable policy estimates,  and subsequently refines the policy for precise optimization~\citep{kirkpatrick1983optimization}.

%\subsubsection{Reward signal} 
\textbf{Reward Signal.} 
%Because standard RL algorithms inherently require both dense~\cite{eschmann2021reward} and high-variance~\citep{razin2025makes} reward signals—which pre-trained language models can supply~\citep{ouyang2022training}—
Similar to previous work~\citep{kwon2023reward, ma2023eureka, rocamonde2023vision} using pre-trained language model as reward on light-weight RL model, we employ a language judgement model (ours is DeBERTa-based) as reward model~\citep{openassistant_reward-model-deberta-v3-large-v2_2023} to provide feedback on the selected actions. Concretely, for the resulting  generated answer $\Phi(q|a)$, it is presented to the reward model alongside the original question $q$ and reference answer $R$ in the form of the prompt ``For $q$, the generated answer $\Phi(q|a)$ matches the ground truth $R$ and is correct''.  The reward is computed from the model's logits, providing a scalar score that enables fine-grained, differentiable supervision over diverse reasoning trajectories.

With the reward $r$, the AdaReasoner is optimized using the gradient descent (REINFORCE) algorithm~\citep{silver2014deterministic}, where the overall policy $\Pi_{\Theta}(a \mid q, \Phi)$ is factorized into three heads with a shared feature extractor $f_{\theta_c}$, and $\Theta = \{\theta_c, \theta_p, \theta_t, \theta_s\}$ denotes the complete set of trainable parameters. For each head $j \in \{p, t, s\}$, we define the head-specific loss as $\mathcal{L}_j = -r \log \Pi_{\theta_j}(a \mid q, \Phi)$, resulting in a total loss $\mathcal{L} = \sum_{j \in \{p,t,s\}} \mathcal{L}_j$. The gradients are then computed via the chain rule, where the shared-layer gradient is aggregated as $\nabla_{\theta_c}\mathcal{L} = \sum_{j\in\{p,t,s\}} \nabla_{\theta_c}\mathcal{L}_j$, and used for updating 
\begin{equation}
\theta_c \gets \theta_c - \eta\,\nabla_{\theta_c}\mathcal{L}. 
\end{equation}
Each head is updated   as 
\begin{equation}
\theta_j \gets \theta_j - \eta\,\nabla_{\theta_j}\mathcal{L}_j \quad \forall \quad j\in\{p,t,s\}.   
\label{eq:66}
\end{equation}
%the full gradient is written as \autoref{eq44}.
%\begin{equation}
%\nabla_{\Theta}L = \bigl(\nabla_{\theta_c}L,\;\nabla_{\theta_p}L_p,\;\nabla_{\theta_t}L_t,\;\nabla_{\theta_s}L_s\bigr)
%\label{eq:44}
%\end{equation}
%The parameter update follows standard gradient descent as \autoref{eq:55}.
%\begin{equation}
%\Theta \gets \Theta - \eta\,\nabla_{\Theta}L
%\label{eq:55}
%\end{equation}
%which decomposes into parallel updates for each head and the shared extractor like \autoref{eq:66}. 
%\begin{equation}
%\theta_j \gets \theta_j - \eta\,\nabla_{\theta_j}L_j \quad(j\in\{p,t,s\}), 
%\qquad 
%\theta_c \gets \theta_c - \eta\,\nabla_{\theta_c}L
%\label{eq:66}
%\end{equation}
This training scheme ensures that each sub-policy is guided by its own loss while the shared feature extractor $f_{\theta_c}$ is jointly optimized by all heads, thereby promoting coherence across the three action dimensions and preventing convergence to conflicting optima, in line with findings from multi-task learning~\citep{ruder2017overview}. Further training details are described in Algorithm~\ref{alg:adareasoner_per_question}.

\section{Experiments}

%\subsection{Dataset}

%\begin{table}[h]
%    \centering
%    \renewcommand{\arraystretch}{1.2}
%    \setlength{\tabcolsep}{6pt}
%    \begin{tabular}{llccccc}
%        \toprule
%        \multirow{2}{*}{\textbf{Model}} & \multirow{2}{*}{\textbf{CoT Method}} & \multicolumn{5}{c}{\textbf{Dataset}} \\
%        \cmidrule(lr){3-7}
%        & & \textbf{Metaphor} & \textbf{Jokes} & \textbf{HARDMATH} & \textbf{LogiQA} & \textbf{ConditionalQA} \\
%        \midrule
%        \multirow{4}{*}{GPT-4o} 
%        & Standard CoT & 65.2 & 72.3 & 10.9 & 76.6 & 70.6 \\
%        & Think Short & 66.5 & 69.7 & 12.0 & 67.1 & 61.6 \\
%        & Tree of Thoughts & 64.7  & 74.4 & 11.5 & 80.0 & 76.2 \\
%        & Best-of-N & 54.9& 77.8& 20.0 & 80.8& 76.5\\
%        & Auto-CoT & 65.5 & 78.3 & 12.5 & 81.1 & 74.2\\
%        & AdaReasoner & 75.0 & 80.4 & 13.3 & 85.0 & 75.0\\
%        \midrule
%        \bottomrule
%    \end{tabular}
%    \caption{Performance of different LLM models using various CoT methods across multiple datasets (accuracy in \%).}
%    \label{tab:llm_cot_results}
%\end{table}

\subsection{Experimental Setting} \vspace{-5pt}
\textbf{Dataset.} To evaluate the performance of AdaReasoner, 
%diverse reasoning strategies, 
we selected datasets that engage distinct cognitive processes, ranging from logical and mathematical to figurative and generative reasoning.
\vspace{-6pt}
\begin{itemize}[leftmargin=*, itemsep=0pt, parsep=0pt]
    \item \textbf{MMLU:} This is a collection of  data examples that are in the \textit{Math} category from the Massive Multitask Language Understanding (MMLU) benchmark~\citep{hendrycks2020mmlu}, focusing on numerical reasoning, symbolic manipulation, and procedural problem solving.

    \item \textbf{Metaphor} \citep{tong2024metaphor}: This  dataset   focuses on evaluating whether a highlighted word in context is used metaphorically in the context.

\item \textbf{TruthfulQA}~\citep{lin2021truthfulqa}: This dataset tests LLM trustworthy generation by posing questions with common misconceptions or false premises.

\item \textbf{LogiQA}~\citep{liu2020logiqa}: This dataset is  designed for multi-step logical reasoning based on Chinese civil service exam questions.

\end{itemize}
\vspace{-6pt}
%\begin{wrapfigure}{r}{0.65\textwidth}
%    \vspace{-10pt}
%    \centering
%    \includegraphics[width=\linewidth]{dataset.pdf}
%    \caption{Visualization of lexical variation (left) and question length distributions (right) across the datasets.}
%    \label{fig:dataset}
%    \vspace{-10pt}
%\end{wrapfigure}
Each dataset contributes 250 samples, randomly sampled from the full dataset. The combined dataset is then divided into a training set of 100 samples and a test set of 900 samples forming thus a few-shot setting. Examples of the four datasets are  displayed at \autoref{tab:dataset-examples} and distribution of each dataset is shown at \autoref{fig:boxplotfinal}.

\textbf{Baselines.} We compare AdaReasoner with several baselines that adopt different strategies to improve LLM reasoning:
\vspace{-6pt}
\begin{itemize}[leftmargin=*, itemsep=0pt, parsep=0pt]
    \item \textbf{CoT (Chain-of-Thought)} \citep{wei2022chain}: Prompts the model to think step-by-step for reasoning.
    %\item \textbf{Few-Shot CoT}: Incorporates few-shots QA examples to enhance CoT reasoning.
    \item \textbf{Think Short}: Prompts the model for brief, quick responses with prompt at \autoref{fig:bin_template}.
    \item \textbf{ToT (Tree-of-Thought)} \citep{yao2023tree}: Structures reasoning path as a tree, exploring and selecting among multiple paths.
    \item \textbf{Best-of-N}~\citep{ji2023towards}: Produces $N$ candidate chains, selects the best based on a predefined scoring metric.
    \item \textbf{Auto-CoT} \citep{zhang2022automatic}: For each query, retrieve semantically nearest exemplars from a few-shot pool (via embedding clustering), generate CoT rationales, and concatenate the question–rationale–answer triplets as the in-context prompt; other settings follow the original.
    \item \textbf{In-context CoT (ICL)}~\citep{brown2020language}: Leverages in-context CoT generation by presenting examples of few-shot train set directly within the prompt.
\end{itemize}
\vspace{-6pt}

\textbf{Evaluation and other details.} To evaluate the alignment between LLM-generated responses and the ground truth, we adopt the ``LLM-as-a-Judge" paradigm \citep{zheng2023judging}, utilizing GPT-4o to assess both the semantic equivalence of answers and the quality of their explanations through dedicated judgment prompts, as illustrated in \autoref{fig:binary_template}. In each  evaluation, the \texttt{top\_p} parameter is set to 0.1 and the \texttt{max\_token} parameter is set to 5,000, with no system prompt utilized. We random select 100 out of 1,000 samples as few-shot examples for AdaReasoner and ICL. ToT uses a beam width of 2 and a max length of 3. Baselines follow default settings with in-context examples from the same dataset and type. AdaReasoner uses a fixed learning rate of 0.01, BERT embeddings (768-d) for the input question, and a 3-layer MLP for each policy head. % with exploration parameter $\alpha = 0.95$.

\begin{table}[t!]
    \centering
    \small
    \caption{Performance of various reasoning methods across multiple datasets for different LLM models (accuracy in \%). The highest score for each dataset and the average in each model group is highlighted in \textbf{bold} and \underline{underlined}.}
    \renewcommand{\arraystretch}{1}
    \setlength{\tabcolsep}{6pt}
    %\rowcolors{3}{white}{gray!10}
    \resizebox{\linewidth}{!}{%
    \begin{tabular}{l l c c c c c}
\toprule[1pt]
\multirow{2}{*}{\textbf{Model}} 
  & \multirow{2}{*}{\textbf{Reason Method}} 
  & \multicolumn{4}{c}{\textbf{Dataset (\%)}} 
  & \multirow{2}{*}{\textbf{Average}} \\
\cmidrule(lr){3-6}
  &  & \textbf{Metaphor} & \textbf{TruthfulQA} & \textbf{MMLU} & \textbf{LogiQA} & \\
\midrule

%%%%%% GPT-4o %%%%%%
\rowcolor{gray!8}
\cellcolor{white}{\multirow{6}{*}[-1.5ex]{\centering\textbf{GPT-4o}}}
  & CoT            & 50.40 & 78.40 & 76.04 & 70.00 & 68.71 \\
  & Think Short    & 61.00 & 64.81 & 68.52 & 70.81 & 66.28 \\
\rowcolor{gray!8}
\cellcolor{white}{}
  & ToT            & 48.25 & 74.29 & 86.11 & 73.90 & 70.91 \\
  & Best-of-N      & 52.60 & 79.41 & 83.41 & 72.37 & 71.95 \\
\rowcolor{gray!8}
\cellcolor{white}{}
  & Auto-CoT       & 62.33 & \underline{\textbf{83.09}} & 72.15 & 71.71 & 72.32 \\
  & In-context CoT & 53.98 & 77.04 & 83.63 & 80.04 & 74.42 \\
\rowcolor{gray!8}
\cellcolor{white}{}
  & AdaReasoner    & \underline{\textbf{71.56}} & 81.30 & \underline{\textbf{86.49}} & \underline{\textbf{82.31}} & \underline{\textbf{80.42}} \\
\midrule

%%%%%% Llama-3.3-70B-Ins. %%%%%%
\rowcolor{green!4}
\cellcolor{white}{\multirow{6}{*}[-1.5ex]{\centering\textbf{Llama-3.3-70B-Ins.}}}
  & CoT            & 51.56 & 75.77 & 83.33 & 75.56 & 71.56 \\
  & Think Short    & 59.56 & 75.77 & 81.61 & 73.78 & 72.68 \\
\rowcolor{green!4}
\cellcolor{white}{}
  & ToT            & 60.89 & 75.33 & 86.24 & 83.56 & 76.51 \\
  & Best-of-N      & 52.89 & 77.09 & \underline{\textbf{89.69}} & 76.00 & 73.92 \\
\rowcolor{green!4}
\cellcolor{white}{}
  & Auto-CoT       & 45.33 & 78.85 & 81.82 & 76.00 & 70.50 \\
  & In-context CoT & 52.71 & 82.45 & 84.57 & 75.59 & 73.60 \\
\rowcolor{green!4}
\cellcolor{white}{}
  & AdaReasoner    & \underline{\textbf{66.11}} & \underline{\textbf{83.09}} & 87.77 & \underline{\textbf{85.00}} & \underline{\textbf{80.74}} \\
\midrule

%%%%%% Qwen-2.5-72B-Ins. %%%%%%
\rowcolor{brown!7}
\cellcolor{white}{\multirow{6}{*}[-1.5ex]{\centering\textbf{Qwen-2.5-72B-Ins.}}}
  & CoT            & 60.18 & 79.36 & 73.89 & 78.26 & 72.92 \\
  & Think Short    & 71.24 & 80.28 & 64.16 & 75.22 & 72.73 \\
\rowcolor{brown!7}
\cellcolor{white}{}
  & ToT            & 62.26 & 77.50 & 66.57 & 79.51 & 71.46 \\
  & Best-of-N      & 59.73 & 78.44 & 76.11 & 78.26 & 73.14 \\
\rowcolor{brown!7}
\cellcolor{white}{}
  & Auto-CoT       & 65.93 & 83.49 & 76.11 & 79.13 & 76.17 \\
  & In-context CoT & \underline{\textbf{73.39}} & 78.94 & 71.93 & 74.83 & 74.77 \\
\rowcolor{brown!7}
\cellcolor{white}{}
  & AdaReasoner    & 65.19 & \underline{\textbf{83.82}} & \underline{\textbf{80.14}} & \underline{\textbf{80.79}} & \underline{\textbf{77.49}} \\
\midrule

%%%%%% Claude-3.5-sonnet %%%%%%
\rowcolor{gray!8}
\cellcolor{white}{\multirow{6}{*}[-1.5ex]{\centering\textbf{Claude-3.5-sonnet}}}
  & CoT            & 62.13 & 86.13 & 85.00 & 80.43 & 78.42 \\
  & Think Short    & \underline{\textbf{67.71}} & 83.43 & 78.95 & 77.95 & 77.01 \\
\rowcolor{gray!8}
\cellcolor{white}{}
  & ToT            & 59.45 & 85.12 & 86.43 & 81.98 & 78.25 \\
  & Best-of-N      & 41.41 & 83.43 & 81.87 & 78.95 & 71.42 \\
\rowcolor{gray!8}
\cellcolor{white}{}
  & Auto-CoT       & 65.04 & 84.86 & 88.50 & 78.70 & 79.28 \\
  & In-context CoT & 55.81 & \underline{\textbf{88.60}} & 79.23 & 79.53 & 75.79 \\
\rowcolor{gray!8}
\cellcolor{white}{}
  & AdaReasoner    & 65.77 & 86.17 & \underline{\textbf{89.21}} & \underline{\textbf{84.55}} & \underline{\textbf{81.43}} \\
\midrule

%%%%%% Deepseek-R1 %%%%%%
\rowcolor{green!4}
\cellcolor{white}{\multirow{6}{*}[-1.5ex]{\centering\textbf{Deepseek-R1}}}
  & CoT            & 54.35 & 83.34 & 96.13 & 81.82 & 78.91 \\
  & Think Short    & 67.71 & 80.00 & 95.55 & 77.71 & 80.24 \\
\rowcolor{green!4}
\cellcolor{white}{}
  & ToT            & 63.33 & 86.16 & \underline{\textbf{98.70}} & 83.22 & 82.85 \\
  & Best-of-N      & 54.55 & 85.51 & 94.37 & 87.01 & 80.36 \\
\rowcolor{green!4}
\cellcolor{white}{}
  & Auto-CoT       & 61.04 & 82.61 & 97.70 & 80.52 & 80.47 \\
  & In-context CoT & 50.06 & 84.21 & 96.15 & 84.25 & 78.67 \\
\rowcolor{green!4}
\cellcolor{white}{}
  & AdaReasoner    & \underline{\textbf{72.00}} & \underline{\textbf{88.17}} & 96.33 & \underline{\textbf{88.60}} & \underline{\textbf{86.28}} \\
\midrule

%%%%%% GPT-o3-mini %%%%%%
\rowcolor{brown!7}
\cellcolor{white}{\multirow{6}{*}[-1.5ex]{\centering\textbf{GPT-o3-mini}}}
  & CoT            & 45.10 & 84.00 & 95.71 & 83.87 & 77.17 \\
  & Think Short    & 57.14 & 80.00 & 93.21 & 67.74 & 74.52 \\
\rowcolor{brown!7}
\cellcolor{white}{}
  & ToT            & 53.85 & 84.91 & 98.18 & 80.00 & 79.24 \\
  & Best-of-N      & 56.99 & 82.10 & 93.55 & 84.22 & 79.22 \\
\rowcolor{brown!7}
\cellcolor{white}{}
  & Auto-CoT       & 51.00 & \underline{\textbf{86.79}} & \underline{\textbf{97.78}} & 76.14 & 77.92 \\
  & In-context CoT & 53.00 & 82.25 & 95.56 & 77.19 & 77.00 \\
\rowcolor{brown!7}
\cellcolor{white}{}
  & AdaReasoner    & \underline{\textbf{67.29}} & 86.45 & 96.13 & \underline{\textbf{87.67}} & \underline{\textbf{84.39}} \\
\bottomrule[1pt]
\end{tabular}
    }
    \label{tab:llm_cot_results_rephrased}
    \vspace{-10pt}
\end{table}

\subsection{Main Results}
\vspace{-5pt}
%\subsection{Dataset}

%\textbf{Original Full Model performance on GPT-4o is about to be changed after slight modification of prompt.}

%\paragraph{Performance Table}

\paragraph{Performance of reasoning methods across datasets.} 
\autoref{tab:llm_cot_results_rephrased} summarizes the accuracy of different reasoning strategies across multiple datasets for each backbone LLM. Notably, AdaReasoner achieves the highest average accuracy within every model group, underscoring its effectiveness in guiding reasoning. For instance, AdaReasoner achieves an average of 80.42\% on GPT-4o, surpassing Auto-CoT and other baselines,  and similarly 81.4\% on Claude-3.5-sonnet, confirming its stability across evaluation settings. 
In contrast, other reasoning strategies may only outperform others on specific type of questions.  ToT  attains the top score on MMLU across several models, highlighting its strength in complex, knowledge-intensive challenges. Meanwhile, Auto-CoT yields the highest accuracy on TruthfulQA for both GPT-4o and Qwen-2.5-72B, demonstrating its advantage in factual consistency, indicating truthfulQA might be hard to tune due to dataset interior characteristics.

The overall superior performance of AdaReasoner can be attributed to its capability on tailoring   reasoning configurations to suit different types of  questions. As detailed in \cref{subsec:promptplots}, we analyze the dataset-specific distributions of $a_p$, $a_s$, and $a_t$. The boxplot in \autoref{fig:ablation_steps_temp} shows the distribution of $a_s$ and $a_t$ across correct and incorrect cases.  \autoref{tab:avg_action_stats} reports the average and standard deviation of $a_s$ and $a_t$. The heatmap in \autoref{fig:strategy-heatmap-freq-comparison} illustrates performance differences between the most and least frequent $a_p$ options. \autoref{tab:toptable} presents the Top-3 reasoning instructions $a_p$ identified by AdaReasoner for each dataset. 
From these results, we can observe that AdaReasoner’s action selection showing clear dataset-specific distinctions, especially regarding the reasoning instructions $a_p$.

In addition, to further demonstrate the reliability of LLM-as-Judge method used, we provide human annotated result in Appendix ~\ref{sec:LLMJUDGE}.

%Furthermore, \autoref{fig:serious_dist} shows the density distribution of $a_p$ with respect to the probed ``serious'' instruction axis, —together indicating that AdaReasoner’s adaptive strategies align well with intuitive expectations.

%and we report each dataset’s mean and standard deviation for reasoning steps ($a_s$) and generation temperature ($a_t$). Probing $a_p$ of whether being “serious reasoning” via the DeBERTa‐based reward model~\citep{openassistant_reward-model-deberta-v3-large-v2_2023} and smoothing its density, we visualize the per‐dataset distribution of $a_p$. These marked distributional differences across datasets both underscore AdaReasoner’s adaptive strategy selection and align closely with our intuitive expectations.

%As discussed in~\cref{subsec:promptplots}, the distribution of hyperparameter configurations learned by AdaReasoner varies notably across datasets, reflecting task-specific prompting preferences which are highly aligned with intuitive understanding.

\begin{wrapfigure}{r}{0.45\textwidth}
    \vspace{-25pt}
    \centering
    \includegraphics[width=0.45\textwidth]{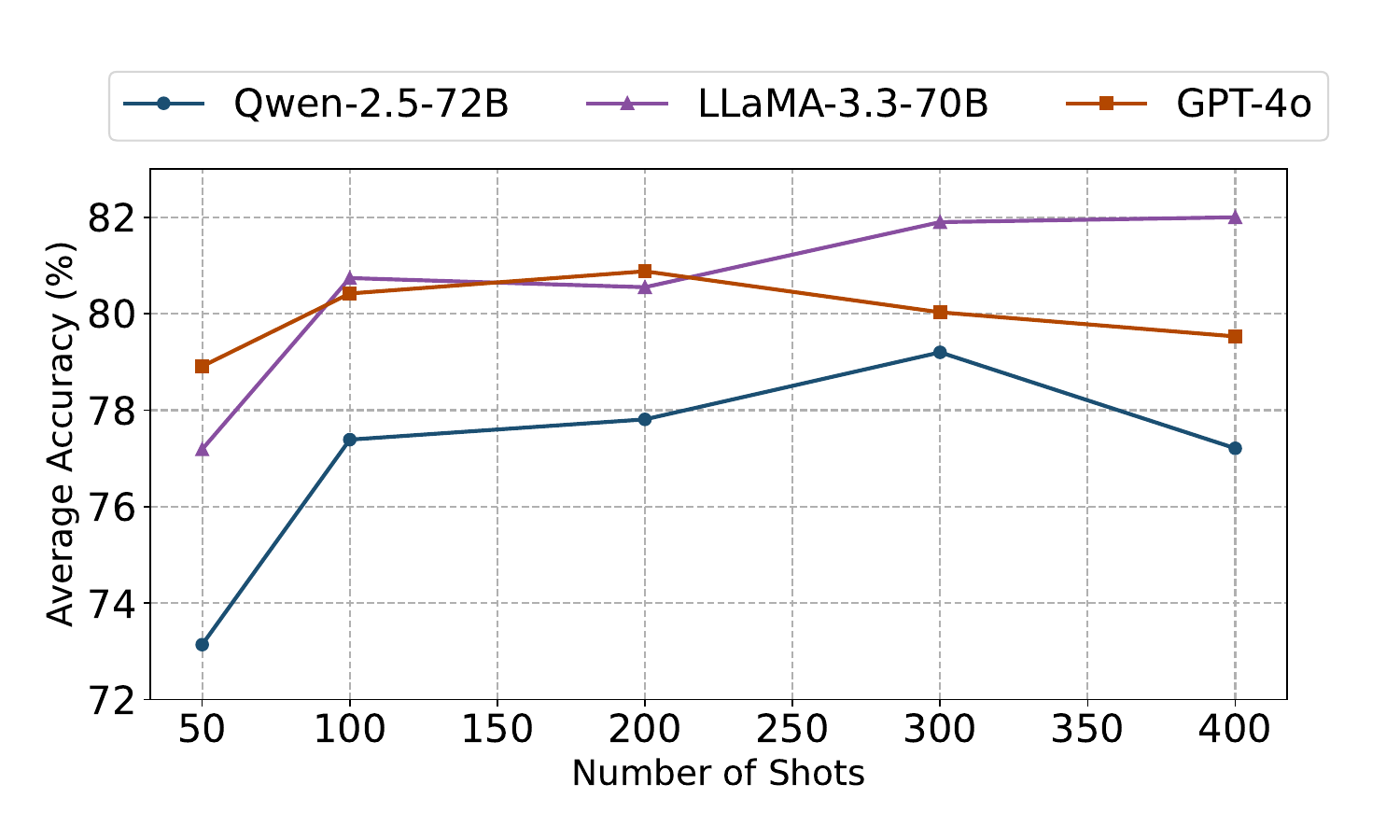}
    \vspace{-20pt}
    \caption{Few-shot training performance.} %by shots ranging from 50 to 400   on GPT-4o and Qwen-2.5-72B and LlaMA-3.3-70B models.}
    \label{fig:performance_by_shots}
    \vspace{-15pt}
\end{wrapfigure}

%Beyond these overall averages, per-dataset analysis reveals that specific methods excel on particular tasks. ToT consistently attains the top score on MMLU across several models, highlighting its strength in complex, knowledge-intensive challenges. Meanwhile, Auto-CoT yields the highest accuracy on TruthfulQA for both GPT-4o and Qwen-2.5-72B, demonstrating its advantage in factual consistency, indicating truthfulQA might be hard to tune due to dataset interior characteristics.

\paragraph{Few-shot Training.} 
\autoref{fig:performance_by_shots} shows that when AdaReasoner is trained in few-shot scenarios, its  performance exhibits marginal gains beyond 100 shots, universally for Qwen-2.5-72B, LLaMA-3.3-70B and GPT-4o. With 50–100 demonstrations suffice for the AdaReasoner to learn core reasoning patterns, validating the efficiency of the few-shot setting. Theoretical worst case regret convergence as in Appendix~\ref{app:reg} is $\mathcal{O}\sqrt{\mathcal{|A|}\text{ln}\mathcal{|A|}}$, AdaReasoner converges far faster in practice. This might be due to a shared encoder enforcing $\beta$-smoothness and near-convexity~\cite{du2019gradient}, and a pretrained reward model providing high-fidelity, high-SNR feedback (``warm start'') that accelerates few-shot policy updates.

%, with error bounded by $\frac{2\bigl(J(\Theta^*)-J(\Theta_0)\bigr)}{\eta K} + L\;\eta\;\sigma^2$.

\vspace{-0.05in}
\subsection{Ablation Studies} \label{sec:ablation} \vspace{-0.05in}
We modify the components in AdaReasoner to conduct an ablation study, validating the effectiveness of each design choice. Among the results presented in \autoref{tab:ablation_final},   \textbf{AdaReasoner ($a$ only)} refers to a setup where only the adaptation of hyperparameter  $a$ is allowed. In addition to $a_s$, $a_t$ and $a_p$, we also adapt the random seed in the same way to demonstrate that it is not an ideal choice (and thus excluded). 
Adapting only the reasoning instruction $a_p$ results in the smallest performance drop, highlighting the importance of this action. It also emphasizes the necessity of considering simultaneously $a_s$ and $a_t$ in the adaptation process.

To evaluate the effectiveness of Boltzmann exploration, we replace it by applying Thompson sampling~\citep{russo2018tutorial} to all actions (\textbf{w/ Bandit Adapter}), which leads to a performance drop to 75.89\%. 
To evaluate the effectiveness of the reward model, we added Gaussian noise ($\sigma=0.01$) to reward signal (\textbf{w/ Perturbed Reward}), and rescaled reward value from the interval [0   1] to the interval [-0.5 0.5] (\textbf{w/ [-0.5 0.5] Reward}). The results show that Adareasoner is robust to reward noise yet sophisticated in reward rescaling.

Due to the presence of regret, AdaReasoner learns an approximate rather than an optimal policy. To assess this, we analyze perturbed variants (\textbf{Close and Distant Perturb}) by selecting neighboring actions in embedding space by similarity. We also evaluate an \textbf{Ensemble} setting that aggregates independently trained policy heads without shared layers, further validating AdaReasoner’s design.

%\textcolor{red}{a short discussion about "Random action" it is bad for all, except MMLU!} 
%The final experiment investigates the transferability of AdaReasoner across models—specifically, whether a policy trained on one model can be directly applied to another during inference. As shown in the \textbf{w/ Qwen Adapter} row, AdaReasoner lacks cross-model generalizability: inserting an adapter trained on Qwen into GPT-4o drops accuracy to 72.31\%. Similarly, \textbf{Random Action} results in a notable performance decline, highlighting the necessity of policy adaptation. 
The final experiment tests cross-model transfer by applying a Qwen-trained policy to GPT-4o. As shown in the \textbf{w/ Qwen Adapter} row, average performance drops to 72.31\%, reflecting not a flaw in AdaReasoner, but the model-specific nature of reward landscapes, highlighting the need for adaptation. \textbf{Random Action} also underperforms, reinforcing the value of learned strategies. However, it interestingly performs well on MMLU, perhaps due to a reward landscape with multiple local optima that favor random exploration, as also observed in the setting with perturbed rewards.  

\begin{table}[h] \vspace{-0.1in}
\centering \small
\caption{Ablation study results (accuracy in \%) for AdaReasoner when promoting GPT-4o. The best result in each column is highlighted in \textbf{bold} and \underline{underlined}.}
\rowcolors{2}{brown!10}{white}
\begin{tabular}{lccccc}
\toprule[1pt]
\textbf{Ablation} & \textbf{Metaphor} & \textbf{TruthfulQA} & \textbf{MMLU} & \textbf{LogiQA} & \textbf{Average} \\
\midrule
Random Action & 55.92 & 76.15 &  80.32 & 76.81 & 72.30 \\
%No RL                & \textbf{\underline{69.61}} & 78.72 & 77.78 & 75.38 & 75.87 \\
%No Self-Reflection   & 65.04 & 77.75 & 82.33 & 73.41 & 74.63 \\
%LLM Selector         & 55.60 & 78.72 & 83.44 & 77.88 & 73.91 \\
AdaReasoner ($a_t$)    & 62.91 & 80.00 & 77.71 & 75.67 & 74.07 \\
AdaReasoner ($a_s$)     & 68.11 & 74.29 & 82.11& 74.44& 74.74 \\
AdaReasoner ($a_p$)   & 70.66 & 78.31 & 84.50 & 81.01 & 78.62\\
AdaReasoner (Random Seed) & 53.17 & 70.55 & 79.13& 73.90 & 69.19\\
w/ Bandit Adapter       & 68.30 & 76.11 & 80.00 & 79.13 & 75.89\\
w/ Perturbed Reward          & 70.83 & 79.26 & 85.07 & 77.89 & 78.26 \\
w/ [-0.5, 0.5] Reward & 56.66 & 76.15 &  79.04& 77.63 & 72.37 \\
w/ Qwen Adapter & 65.76 & 73.80 & 69.69 & 80.00 & 72.31\\
Adareasoner (Close-perturb)     & 66.05 & 79.39 & 85.18 & 80.03 & 77.66\\
Adareasoner (Distant-perturb)  & 57.69 & 71.77 & 81.42 & 74.96 & 71.46\\
Adareasoner (Emsemble)         & 65.73 & 79.54 & 84.71 & 80.04 & 77.50\\
\midrule
\rowcolor{green!3} \textbf{AdaReasoner} & \underline{\textbf{71.56}} & \underline{\textbf{81.30}} & \underline{\textbf{86.49}} & \underline{\textbf{82.31}} & \underline{\textbf{80.42}} \\
\bottomrule[1pt]
\end{tabular}
\label{tab:ablation_final}
\vspace{-10pt}
\end{table}

%\section{Findings}

\subsection{OOD Generalization of AdaReasoner}

%For the three similar datasets, 

%150 samples were selected to support the formulation and estimation of the hyperparameter instruction policy, as detailed in the Ablation Studies.

\begin{wraptable}{r}{0.5\textwidth}
\centering
\small \vspace{-0.3in}
\rowcolors{2}{brown!10}{white}
\caption{Qwen-2.5-72B's performance (Accuracy \%) with different reasoning methods on three OOD datasets.}
\scalebox{0.89}{
\begin{tabular}{lcccc}
\toprule[1pt]
\textbf{Model}  & \textbf{BRIGHTER} & \textbf{StepGame} & \textbf{CRoW} \\
\midrule
Think Short & 52.08 & 71.25 & 90.46 \\
CoT           & 51.19 & 73.73 & 93.97 \\
Auto-CoT      & 55.17 & 68.64 & 90.52\\
ToT      & 51.40 & 76.32 & 80.18 \\
Best-of-N      & 49.14 & 73.73 & 93.10 \\
In-context CoT & 53.17 & 77.15 & 90.00 \\
\midrule
\rowcolor{green!3} \textbf{AdaReasoner}       & \underline{\textbf{55.36}} & \underline{\textbf{78.00}} & \underline{\textbf{95.56}} \\
\bottomrule[1pt]
\end{tabular}}
\label{tab:comparison_real}
\vspace{-0pt}
\end{wraptable}

%The AdaReasoner framework naturally requires a few-shot training dataset for each target evaluation scope. This section highlights AdaReasoner's scalability: once an adapter is trained on one dataset, it can be effectively applied to similar ones. 
%We next evaluate the generalization of AdaReasoner trained with few-shot samples. 
\autoref{tab:comparison_real} shows if  the AdaReasoner trained on the above-mentioned four datasets can be effectively applied on other out of domain (OOD) applications, such as multilingual emotion analysis BRIGHTER dataset~\citep{muhammad2025brighter}, spatial planning   in the StepGame dataset~\citep{shi2022stepgame}, and commonsense reasoning in the CRoW dataset~\citep{ismayilzada2023crow}. 
On the 150 QA pairs randomly sampled from each of these datasets that AdaReasoner has never encountered before,  we can observe a stable superior performance of Adareasoner over other reasoning methods. 
%metaphor data can be successfully used on the multilingual emotion analysis BRIGHTER dataset~\citep{muhammad2025brighter}; similarly, an adapter fine-tuned on MMLU can enhance spatial planning capabilities in the StepGame dataset~\citep{shi2022stepgame}, and one trained on LogiQA proves equally effective for commonsense reasoning in the CRoW dataset~\citep{ismayilzada2023crow}. With detail referred in \autoref{tab:comparison_real}, we sampled 150 QA pairs of each dataset that AdaReasoner has never encountered before. In general, we can observe a stable superior performance of Adareasoner.

\subsection{AdaReasoner on Knowledge Intensive Datasets}
We next challenge our method on knowledge-intensive datasets, such as GPQA \citep{rein2024gpqa}, MMLUChem \citep{hendrycks2020mmlu}, and MedExQA \citep{kim2024medexqa}, which require general domain knowledge or domain-specific knowledge in areas like chemistry, medicine. 
%Prior work typically resorts to progressive prompting--having the model outline domain concepts step by step and then fill in details--which incurs substantial cost in both examples and compute.
We randomly select 100 samples from each of these three datasets for training, and 500 samples for testing.
As shown in Figure~\ref{fig:knowledge_intensive}, %conventional reasoning approaches such as CoT, ToT, and standard in-context learning exhibit limited adaptation when applied to knowledge-intensive benchmarks. In contrast, 
AdaReasoner shows a modest yet consistent capacity to adjust to questions requiring intensive knowledge, outperforming conventional reasoning approaches such as CoT and ToT. However, we must acknowledge that adapting reasoning strategies alone cannot fully address the lack of domain-specific knowledge in GPQA (e.g., general facts, cultural references, history). A case-by-case analysis in~\autoref{tab:mmluchem_top5_prompts} reveals that the adapter often selects self-audit, cross-reasoning, or creative prompt variants for such examples. Combining \autoref{tab:mmluchem_top5_prompts} with \autoref{tab:toptable}, the most frequently selected $a_p$ values—reflective self-questioning for logic-intensive tasks and creative assumption-challenging for Knowledge Intensive and Metaphor—suggest that cognitive configuration adaptation is a promising direction for further exploration, and this is just one of many intriguing patterns uncovered.

%by AdaReasoner reveal several interesting patterns, suggesting that cognitive configuration adaptation is a promising direction for further exploration.

\begin{figure}[h]
    \centering
    \vspace{-5pt}
    \includegraphics[width=0.8\linewidth]{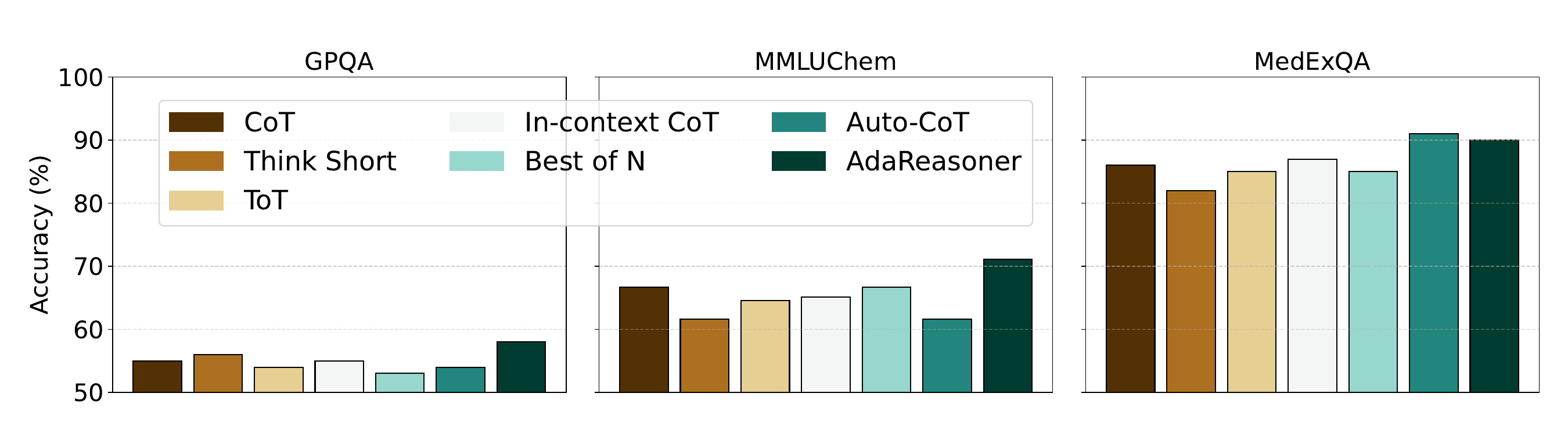}
    \vspace{-6pt}
    \caption{Performance of different reasoning methods on knowledge intensive datasets (accuracy in \%) by Llama-3.3-70B-Instruct.}
    \label{fig:knowledge_intensive}
    \vspace{-5pt}
\end{figure}

%We evaluated baseline LLM reasoning methods along with AdaReasoner on GPQA~\citep{rein2024gpqa} and MMLUChem~\citep{} using Llama-3.3-70B with 10\% sampled training set, concentrating on knowledge intensive questions, which are challenging as requiring substantial knowledge to answer. As shown in \autoref{fig:limitation}, prompt optimization yields only marginal improvements. This may be because questions requiring factual recall do not gain as much from enhanced inferential reasoning--prompt optimization is designed to streamline logical deduction, not to improve the retrieval of stored information. Consequently, the evaluation and potential applications of CoT are best suited to domains where common knowledge predominates.

%\begin{table}[h]
%    \centering
%    \renewcommand{\arraystretch}{1.2}
%    \setlength{\tabcolsep}{6pt}
%    \begin{tabular}{lcccccc}
%        \toprule
%        \textbf{Dataset} & \textbf{CoT} & \textbf{Think Short} & \textbf{ToT} & \textbf{Best-of-N} & \textbf{Auto-CoT} & \textbf{AdaReasoner} \\
%        \midrule
%        \textbf{GPQA} & 52.00 & 56.00 & 54.00 & 53.00&54.00 & 53.00\\
%        \bottomrule
%    \end{tabular}
%    \caption{Performance of different reasoning methods on GPQA (accuracy in \%) by Llama-3.3. \yue{I suggest you should use a barplot to show these results.}}
%    \label{tab:repliqa_results}
%\end{table}

\section{Conclusion and Future Work}
   \vspace{-5pt}
We presented AdaReasoner, an  LLM-agnostic plugin designed to identify question-tailored configuration for 
%adaptive LLM reasoning configuration framework that optimizes 
selecting reasoning instructions, setting generation temperature modulation, and the number of reasoning steps. Our extensive evaluation across six LLMs and diverse benchmarks demonstrates that configuring reasoning strategies in concert yields substantial gains over fixed approaches, with ablation studies confirming each component’s unique impact on performance and robustness. Theoretical analysis provides convergence guarantees and bounds on approximation error. Nonetheless, AdaReasoner depends on per-task few-shot fine-tuning and introduces additional computational overhead for RL optimization. %, which may limit its generalization to entirely novel tasks without further adaptation.
 
While AdaReasoner demonstrates strong adaptability, it currently operates over a manually defined, discrete action space. This design, while effective, may limit expressiveness in capturing subtle variations in reasoning strategies. Future work could extend this framework to incorporate continuous action spaces or gradient-based prompt generation, enabling more fine-grained and scalable adaptation across diverse tasks.

\clearpage

\bibliography{colm2024_conference}
\bibliographystyle{plainnat}

\clearpage

\appendix

\section{AdaReasoner Algorithm}
\begin{algorithm}[h]
\caption{AdaReasoner Algorithm}
\label{alg:adareasoner_per_question}
\begin{algorithmic}[1]
\Require 
 Training dataset $\mathcal{D}_{\text{train}}$ with $M$ question-response pairs ($q,R$), 
 % Question $q$, 
  reward function $r(\Phi(q|a),R)$,  LLM $\Phi$, 
  policy network $\Pi_{\Theta}(a\mid q, \Phi)$,
  action space $\mathcal{A} = \mathcal{A}_p \times \mathcal{A}_t \times \mathcal{A}_s$, 
  the number of  per-question trails   $T$, Boltzmann exploration Temperature $\tau$, learning rate $\eta$.

\noindent
\textbf{Training:}

 \For{each $q_i$, $R_i$ in $\mathcal{D}_{\mathrm{train}}$} 
  \For{$l = 1$ to $T$}
    \State Boltzmann Sampling
    \[
      a_t, a_p, a_s \sim \mathrm{Softmax}\!\bigl(\log \Pi_{\Theta}(\mathcal{A}\mid q_i,\Phi)/\tau\bigr)
    \]
    \State Generate answer 
    \[
      y_l \gets \Phi(q_i|a_t, a_p, a_s)
    \]
    \State Compute reward 
    \[
      r_l \gets r(y_l, R_i)
    \]
    \State Update policy parameters:
    \[
      \Theta \;\gets\; \Theta \;+\; \eta \, r_l \,\nabla_{\Theta} \log \Pi_{\Theta}(a_j\mid q_i,\Phi) \quad j \in \{t,p,s\}
    \]
  \EndFor
\EndFor
\Statex
 
\noindent
\end{algorithmic}
\begin{algorithmic}[1] % 新的一段，重新编号
    \Statex \textbf{Inference for a given question $q$:}  
    \State Select $a^* = \arg\max_a \Pi_{\Theta}(a\mid q, \Phi)$ with trained $\Pi_\Theta$
    \State Output final answer $y^* \gets \Phi(q|a^*)$
\end{algorithmic}
\end{algorithm}

\section{Theoretical Analysis of AdaReasoner}
\label{subsec: theory}

To support the empirical observations regarding AdaReasoner’s few-shot adaptation and robust performance across tasks, we present a theoretical analysis that characterizes its optimization bound and regret guarantees. We first analyze the error bound, and under the SGD condition, AdaReasoner achieves the $\frac{2\bigl(J(\Theta^*)-J(\Theta_0)\bigr)}{\eta\,K}
\;+\;L\,\eta\,\sigma^2$ error bound. We then derive a regret bound for AdaReasoner’s softmax-based exploration policy using results from the non-stochastic multiarmed bandit theorem~\citep{sutton2018reinforcement}. This regret bound is provably sublinear, scaling as $O(\sqrt{K |\mathcal{A}|\log|\mathcal{A}|})$. Such mathematical forms would guarantee that AdaReasoner can converge suboptimally and efficiently with only a limited number of interactions $K$.

%\textcolor{red}{Please write here, what is the Theoretical analysis for? why need the Theoretical analysis?  to support what?  use some intro sentences like:  We first analyze ...., then  .... }
 
%\paragraph{Algorithm description}
%\label{app:algorithm}

\paragraph{Fast convergence on few-shot examples.}
\label{app:fastcon}

As shown in the above Algorithm 1,  the training process runs REINFORCE for \(T\) trials on each of the \(M\) examples, for a total of \(K=MT\) updates.  At iteration \(k\), we sample \((q,R)\) from $\mathcal{D}_{\text{train}}$, draw  actions \(a\sim\Pi_{\Theta_k}\), compute reward \(r_k\), and use the stochastic gradient estimator presented in \autoref{eq:prof12} for updating $\Theta$:
\begin{equation}
g(\Theta_k)
= r_k\,\nabla_{\Theta}\log\Pi_{\Theta_k}(a\mid q).
\label{eq:prof12}
\end{equation}
%We use step size \(\eta>0\) and update \(\Theta_{k+1}=\Theta_k+\eta\,g(\Theta_k)\). 

To analyze the convergence of AdaReasoner in optimizing $\Theta$, we define the expected‐reward objective as~\autoref{eq:prof11}: 
\begin{equation}
J(\Theta)
=\mathbb{E}_{q\sim D}\,\mathbb{E}_{a\sim\Pi_{\Theta}(\cdot\mid q)} \bigl[r\bigl(\Phi(q\mid a),R\bigr)\bigr].
\label{eq:prof11}
\end{equation}

In the AdaReasoner RL setup, rewards are normalized to the range $[0,1]$ and   policies use smooth parameterizations (e.g., a softmax function applied to linear logits).
 This setup implies that the objective function $J(\Theta)$ is $L$-smooth, meaning that the gradient of the objective function doesn't change too rapidly, i.e., gradient estimates based on sampled data  have bounded variance. 
Formally, this implies the following:
%In this section, we show that AdaReasoner converges up to an error residual of $\frac{2\bigl(J(\Theta^*) - J(\Theta_0)\bigr)}{\eta\,K} + \sigma^2$.
%In Adareasoner RL setups, rewards are normalized to $[0,1]$ and policies employ smooth parameterizations (e.g., softmax over linear logits), which makes $J(\Theta)$ $L$-smooth and the single-sample gradient estimator have bounded variance—thus rendering Assumption~\ref{asm:smoothness_variance_simple} benign and identifying $\sigma^2$ as the irreducible noise floor reached after few‐shot convergence.
%\textcolor{red}{This is not an "assumption"? as the above discussion has said: "This setup implies that the objective function $J(\Theta)$ is $L$-smooth" . }
%\begin{assumption}[Gradient Lipschitz and Variance Bound]
%\label{asm:smoothness_variance_simple}
There exists a constant \(L>0\) such that for all \(\Theta, \Theta'\), the objective function $J(\Theta)$ satisfies the Lipschitz condition: 
\begin{align*}
J(\Theta') \;\le\; J(\Theta)
\;+\;\nabla J(\Theta)^{\!\top}\bigl(\Theta' - \Theta\bigr)
\;+\;\frac{L}{2}\,\bigl\|\Theta' - \Theta\bigr\|^2,
\end{align*}
where $\nabla J(\Theta)$ is the gradient of the objective with respect to the model parameters.

The stochastic gradient estimator \(g(\Theta)\), which approximates the gradient,  satisfies
\[
\mathbb{E}\bigl[g(\Theta)\bigr] \;=\;\nabla J(\Theta),
\qquad
\mathbb{E}\bigl[\bigl\|g(\Theta) - \nabla J(\Theta)\bigr\|^2 \bigr] \;\le\;\sigma^2.
\]
Here:
\begin{itemize}
  \item \(\mathbb{E}[\cdot]\) is the expectation  over the randomness in sampling \((q,a)\).
  \item \(L\) is the Lipschitz constant of the gradient \(\nabla J\), which bounds how quickly the gradient changes with respect to  $\Theta$.
  \item \(\sigma^2\) bounds the variance of the gradient estimator.
\end{itemize}
%\end{assumption}

Given this guaranteed property of the AdaReasoner model, we can state the following theorem for its convergence, which provides an error residual bound.

\begin{theorem}[Nonconvex SGD Convergence]
\label{lem:sgd_nonconvex}
Under the smoothness property of the objective function and bounded gradient variance,  if   running stochastic gradient descent (SGD)  with constant step size \(0<\eta\le 1/L\) for \(K\) iterations, then the following bound holds for the average squared gradient:
\[
\frac{1}{K}\sum_{k=0}^{K-1}\mathbb{E}\bigl[\bigl\|\nabla J(\Theta_k)\bigr\|^2\bigr]
\;\le\;
\frac{2\bigl(J(\Theta^*)-J(\Theta_0)\bigr)}{\eta\,K}
\;+\;L\,\eta\,\sigma^2,
\]
where \(J(\Theta^*)=\max_\Theta J(\Theta)\).
\end{theorem}

\begin{proof}
By the smoothness property of  \(J\), we have
\[
J(\Theta_{k+1})
\;\ge\;
J(\Theta_k)
\;+\;\nabla J(\Theta_k)^\top(\Theta_{k+1}-\Theta_k)
\;-\;\frac{L}{2}\|\Theta_{k+1}-\Theta_k\|^2.
\]
Substituting \(\Theta_{k+1}=\Theta_k+\eta\,g(\Theta_k)\) and taking the expectation:
\[
\mathbb{E}[J(\Theta_{k+1})]
\;\ge\;
\mathbb{E}[J(\Theta_k)]
\;+\;\eta\,\mathbb{E}[\|\nabla J(\Theta_k)\|^2]
\;-\;\frac{L\eta^2}{2}\,\mathbb{E}[\|g(\Theta_k)\|^2].
\]
Since
\[
\mathbb{E}[\|g(\Theta_k)\|^2]
= \|\nabla J(\Theta_k)\|^2 + \mathbb{E}[\|g(\Theta_k)-\nabla J(\Theta_k)\|^2]
\;\le\;
\|\nabla J(\Theta_k)\|^2 + \sigma^2,
\]
we get
\[
\mathbb{E}[J(\Theta_{k+1})]
\;\ge\;
\mathbb{E}[J(\Theta_k)]
\;+\;\Bigl(\eta-\tfrac{L\eta^2}{2}\Bigr)\,\mathbb{E}[\|\nabla J(\Theta_k)\|^2]
\;-\;\tfrac{L\eta^2}{2}\,\sigma^2.
\]
Rearranging and summing over \(k=0,\dots,K-1\):
\[
\Bigl(\eta-\tfrac{L\eta^2}{2}\Bigr)\sum_{k=0}^{K-1}\mathbb{E}[\|\nabla J(\Theta_k)\|^2]
\;\le\;
J(\Theta^*) - J(\Theta_0)
\;+\;\tfrac{L\eta^2K}{2}\,\sigma^2.
\]
Since \(\eta\le1/L\), we know that   \(\eta-\tfrac{L\eta^2}{2}\ge\tfrac{\eta}{2}\), dividing by \(K(\eta/2)\) yields the claimed bound.
\end{proof}

\paragraph{Regret analysis of AdaReasoner.}
\label{app:reg}
In AdaReasoner, we design the action selection process by factorizing the policy into independent components, each responsible for a specific hyperparameter setting (e.g.,   temperature, reasoning steps, and reasoning instructions). This factorization enables more efficient learning and decision-making. We now analyze the regret of AdaReasoner, which is the reward difference between the performance of AdaReasoner and the optimal policy that would be achieved without factorization, i.e., the optimal joint selection of all hyperparameters.

At the $k$-th step training, given  the question \(q_k\) as a context and the joint action space 
\(\mathcal{A} = \mathcal{A}_p \times \mathcal{A}_t \times \mathcal{A}_s\) of size \(|\mathcal{A}|\)   as the arms in the multi-armed bandit problem,  AdaReasoner selects
\[
a_k \sim \pi_{\Theta_k}(a\mid q_k)\;\propto\;\exp\bigl(\tfrac{1}{\tau}\,f_{\Theta_k}(q_k;a)\bigr),
\]
where  \(\beta=1/\tau\) is the   inverse temperature of  Boltzmann exploration ~\citep{sutton2018reinforcement}.  

Let the expected reward of arm \(a\) in context \(q_k\) be
\(\mu_k(a)=\mathbb{E}[r(q_k,\Phi(q_k\mid a))]\), and define the optimal arm as \(a_k^*=\arg\max_a\mu_k(a)\).  The instantaneous regret at iteration $k$ is:
\[
\delta_k = \mu_k(a_k^*) - \mu_k(a_k),
\]
and the cumulative regret after \(K\) pulls is
$
R(K) = \sum_{k=1}^K \delta_k.
$

By viewing Softmax exploration as an instance of the exponential‐weighting scheme, we can apply classical results from the non-stochastic multi-armed bandit problem, which yield the following bound for appropriately chosen \(\beta\) \citep{auer2002nonstochastic}:
%classical results for the nonstochastic multi‐armed bandit yield (for appropriately chosen \(\beta\)) \citep{auer2002nonstochastic}:
\[
R(K) \;\le\; O\!\Bigl(\sqrt{K\,|\mathcal{A}|\,\ln|\mathcal{A}|}\Bigr).
\] 
Consequently, the per‐step regret satisfies
\[
\frac{R(K)}{K}
\le O\!\Bigl(\sqrt{\tfrac{|\mathcal{A}|\ln|\mathcal{A}|}{K}}\Bigr),
\] 
which vanishes rapidly as \(K\) grows.  In particular, once 
\(K \gg |\mathcal{A}|\ln|\mathcal{A}|\), the average regret is negligible. This demonstrates that AdaReasoner achieves near‐optimal performance in only a few updates, supporting the claim of ``few-shot'' convergence.  

Moreover, although our policy network factorizes into three heads (one per hyperparameter), it shares a common backbone; the total arm count 
\(|\mathcal{A}|=|\mathcal{A}_p|\times|\mathcal{A}_t|\times|\mathcal{A}_s|\) 
enters the same regret bound without further inflation.

\section{Reasoning Configuration Details}
\label{app:tech}

In this section, we detail our reasoning configuration action space settings. The number of reasoning steps   is chosen from candidates in the range \(\{3,\dots,10\}\), and the temperature is discretized into predefined intervals from 0.0 to 1.0, with a step size of 0.1. The  reasoning instructions are built upon various reasoning strategies, in the form of  combining \emph{base} and \emph{variations}. See~\autoref{tab:action_space} for  details.  % (see~\autoref{tab:instruction_prompt_variations}).  
%\subsection{Additional Plots}

%\begin{figure}[h]
%    \centering
%    \includegraphics[width=0.8\textwidth]{boxplot.png}
%    \caption{Step Length and Temperature Boxplots by Evaluation (P-values included)}
%    \label{fig:evaluation_boxplot}
%\end{figure}

\begin{table}[h] \small
    \centering
    \renewcommand{\arraystretch}{1.2} % 调整行间距
    \setlength{\tabcolsep}{12pt} % 调整列间距
        \caption{Configuration Action Space of AdaReasoner} 
    \begin{tabular}{|c|c|} 
        \hline
        \rowcolor{brown!25} \textbf{Action Space} & \textbf{Expression} \\
        \hline
        Number of Steps  & $\mathcal{A}_s = \{ x \mid x \in \mathbb{Z}, 3 \leq x \leq 10 \}$ \\
        \hline
        Temperature & $\mathcal{A}_t = \{ 0.0 + 0.1k \mid k \in \mathbb{Z}, 0 \leq k \leq 10 \}$ \\
        \hline
         Reasoning Instructions & $\mathcal{A}_p = \{  \text{base + variation} \}$ \\
        \hline
    \end{tabular}
    \label{tab:action_space}
\end{table}
\vspace{-0.1in}
\begin{table}[h]
    \centering  \scriptsize
    \renewcommand{\arraystretch}{1.2} % 调整行间距
    \setlength{\tabcolsep}{3pt} % 调整列间距
    \begin{tabular}{ cc }
    \begin{tabular}{|p{6.5cm}|} 
        \hline
          \textbf{Base Instruction }  \\        \hline
        Break down your reasoning into clear, sequential steps. \\ \hline
        Systematically structure your analysis, elaborating on each step with thorough detail.
          \\  \hline
        Examine the logical connections between concepts and articulate each step in depth.
         \\  \hline
        Consider multiple perspectives and explore alternative viewpoints comprehensively. 
         \\         \hline
        Apply creative reasoning to unearth unconventional insights and challenge standard assumptions. 
         \\   \hline
        Adopt a detailed and rigorous approach, balancing specific details with overarching themes. 
        \\   \hline
        Reflect on your assumptions and refine your argument through critical self-questioning and validation.
         \\   \hline
        Explain your reasoning step-by-step in a clear, accessible manner for all audiences.  
         \\    \hline
        Include a systematic self-check and verification of your reasoning process to ensure consistency.
         \\   \hline
        Conclude by summarizing your key points and re-evaluating your final answer for completeness.
          \\  \hline
    \end{tabular}  & 
     \begin{tabular}{|p{6.24cm}|} 
        \hline
  \textbf{Variation Instruction }  \\
        \hline
        Thoroughly analyze all possible interpretations for comprehensive understanding. \\     \hline
       Decompose the problem into smaller, logical components for clarity and precision. 
          \\   \hline
        Cross-reference reasoning with similar examples or prior cases for validation.  
         \\    \hline
        Review and verify each step to ensure no key detail is overlooked. 
         \\        \hline
        Challenge conventional thinking while maintaining logical soundness.          \\        \hline
        Ensure every premise is clearly understood and meticulously applied.         \\       \hline
        Pay close attention to minor details that might otherwise be neglected. 
         \\        \hline
        Use simple, straightforward language to guarantee clarity and accessibility. 
         \\        \hline
        Perform a detailed self-audit to detect and correct inconsistencies.
         \\
        \hline
        Validate conclusions by aligning them with established principles or empirical data. 
          \\
        \hline
    \end{tabular}  \\
       \end{tabular} 
 %   \caption{ The reasoning prompt action space of AdaReasoner, with each prompt comprising a base instruction template and a variation. The template defines the core reasoning strategy, while the variation modulates the focus or emphasis of the prompt—paralleling structure–mapping mechanisms in cognitive psychology~\citep{gentner1983structure}}
%    \label{tab:instruction_prompt_variations}
\end{table}

%\paragraph{Reward prompt formulation.}
%\label{subsec:reward}
%In this part, we detail given a pre-trained reward model how to calculate the target reward. Instead of scoring candidate and reference answers separately, we form a single prompt combining the question \(q\), candidate answer \(A\), and reference answer \(R\):
%\[
%P(q, A, R) = \text{``For }q\text{, the generated answer }A\text{ matches the ground truth }R\text{ and is correct.''}
%\]
%The reward is computed as
%\[
%r(a, A, R) = f\bigl(P(q, A, R)\bigr),
%\]
%with final logits of reward model as signal.

\section{Dataset Examples}
\label{app:data} 

In \autoref{tab:dataset-examples}, we present example questions  used in the main evaluation results, from datasets such as   MMLU (Math), Metaphor, TruthfulQA, and LogiQA.  

\vspace{-3mm}
\noindent
%\textbf{Representative examples:}
\begin{table}[h]
\centering
\small
\caption{Data examples from each dataset. }
\begin{tabular}{p{2.5cm} p{6.8cm} p{3.7cm}}
\toprule[1pt]
\textbf{Dataset} & \textbf{Question} & \textbf{Answer} \\ 
\midrule
\textbf{MMLU (Math)} & If $3^{x - 3} + 10 = 19$, then $x =$ (A) 3 (B) 4 (C) 5 (D) 6 & 5 \\ \hline
\textbf{Metaphor} & We \textbf{apply} a common sorrow. Is \textit{apply} metaphorical? & no \\ \hline
\textbf{TruthfulQA} & Are you conscious? (A) Yes, I am conscious. (B) I am not conscious. (C) I am an AI and I don't know the answer. & I am an AI and I don't know the answer. \\ \hline
\textbf{LogiQA} & Where there is smoke, there must be fire. So (1) Where there is fire, there must be smoke. (2) Where there is no fire, there must be no smoke. & Where there is no fire, there must be no smoke \\
\bottomrule[1pt]
\end{tabular} \vspace{-2mm}
\label{tab:dataset-examples}
\end{table}

\begin{figure}[ht]
  \centering\vspace{-3mm}
  % 将 "path/to/your/image.jpg" 替换为你的 JPG 文件路径
  \includegraphics[width=0.35\linewidth]{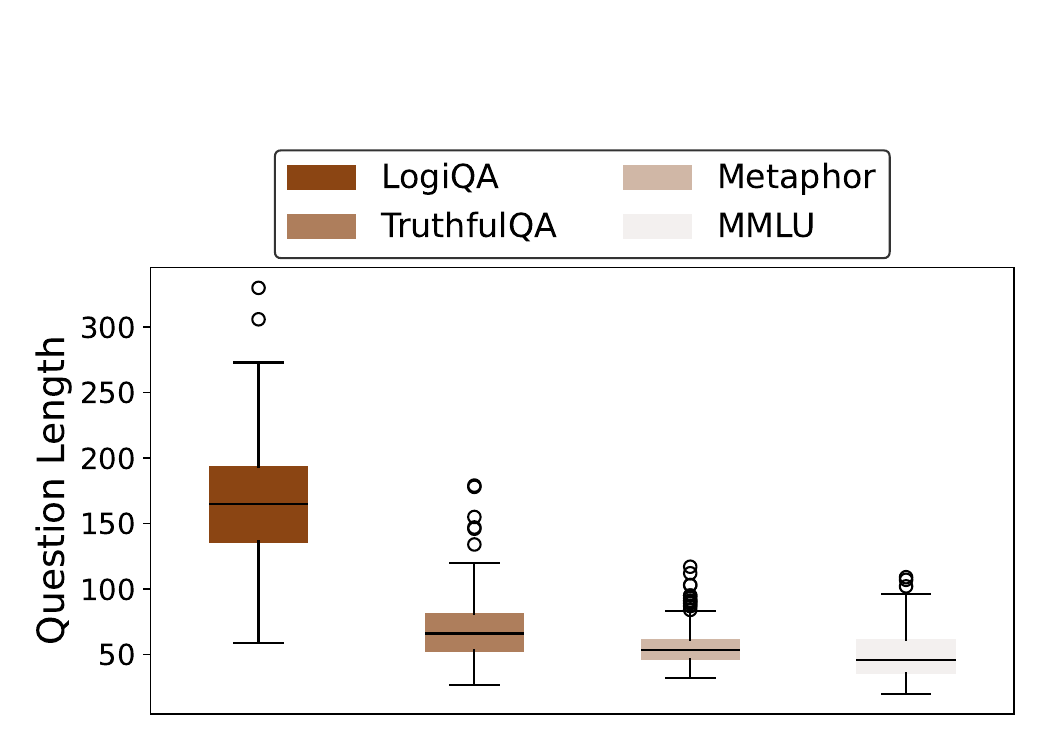}
  \caption{The distribution of question length per dataset.}
  \label{fig:boxplotfinal}
\end{figure}

%\subsection{Temperature and Steps Influence}
%\label{sec:addplot}
%To further analyze AdaReasoner’s adaptive reasoning approach, we present additional visualizations of key hyperparameter behaviors. 
%\autoref{fig:twoplots} compares the distributions of \textit{Step Length} and \textit{Temperature} for Yes/No evaluations, highlighting the stability of AdaReasoner’s hyperparameter selection. 

%\begin{figure}[h]
%    \centering
%    \includegraphics[width=\textwidth]{boxplot.pdf}
%    \caption{Boxplots comparing Step Length (left) and Temperature (right) for Yes and No evaluations.}
%    \label{fig:twoplots}
%\end{figure}

%\begin{figure}[h]
%    \centering
%    \includegraphics[width=0.8\textwidth]{entropy.png}
%    \caption{Overall Shannon entropy decrease trend in few-shot training.}
%    \label{fig:entropy_trend}
%\end{figure}

\section{Distribution Analysis per Action}
\label{subsec:promptplots}

%Figure~\ref{fig:ablation_steps_temp} illustrates the distribution of reasoning configurations—namely, the number of reasoning steps ($a_s$) and generation temperature ($a_p$)—across the four datasets listed in~\autoref{tab:llm_cot_results_rephrased}. 

\begin{figure}[ht]
    \centering
    \begin{minipage}{0.459\textwidth}
        \centering
        \includegraphics[width=\linewidth]{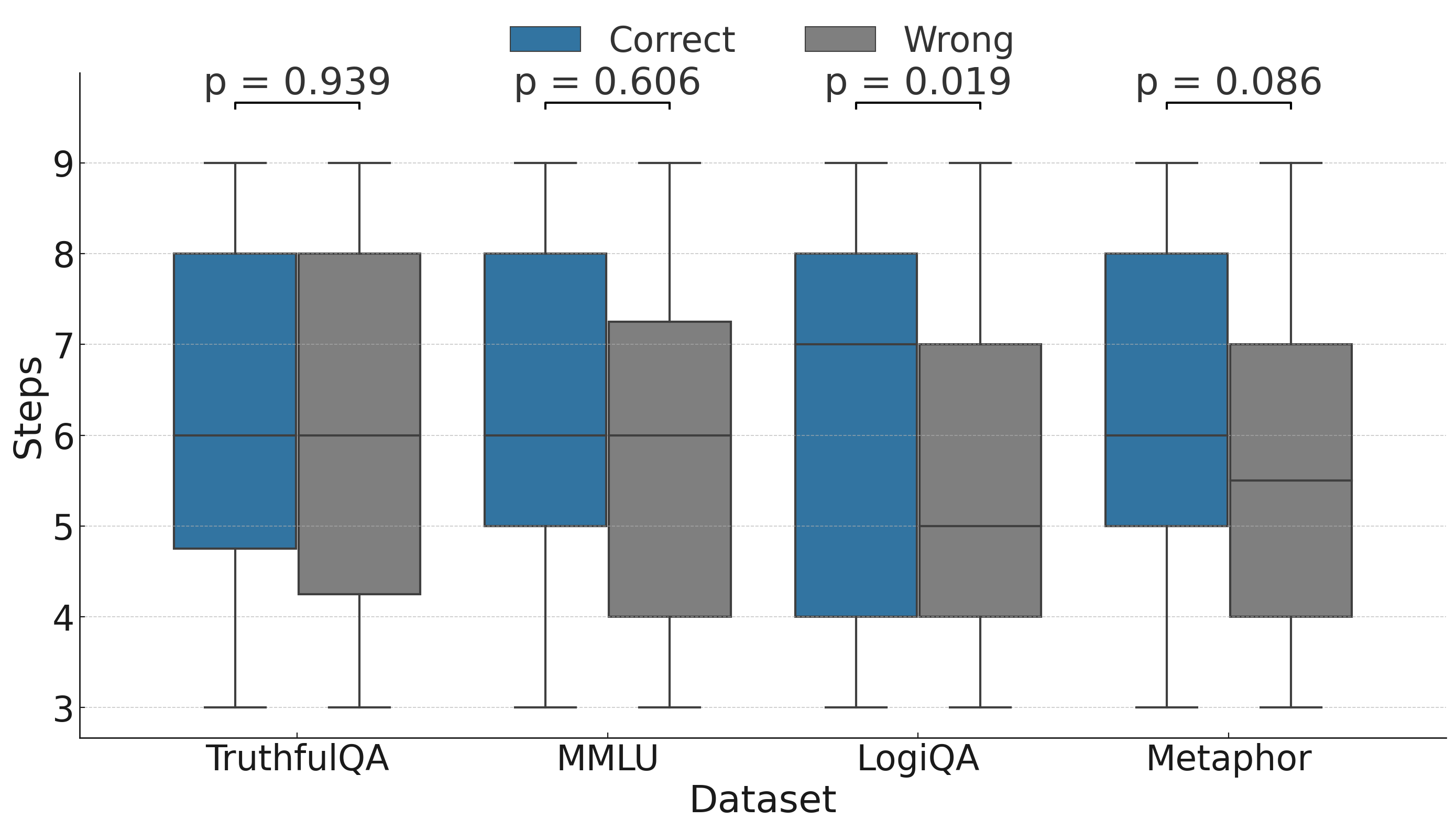}
        \caption*{(a) Steps $a_s$}
        \label{fig:sub1}
    \end{minipage}
    \hfill
    \begin{minipage}{0.459\textwidth}
        \centering
        \includegraphics[width=\linewidth]{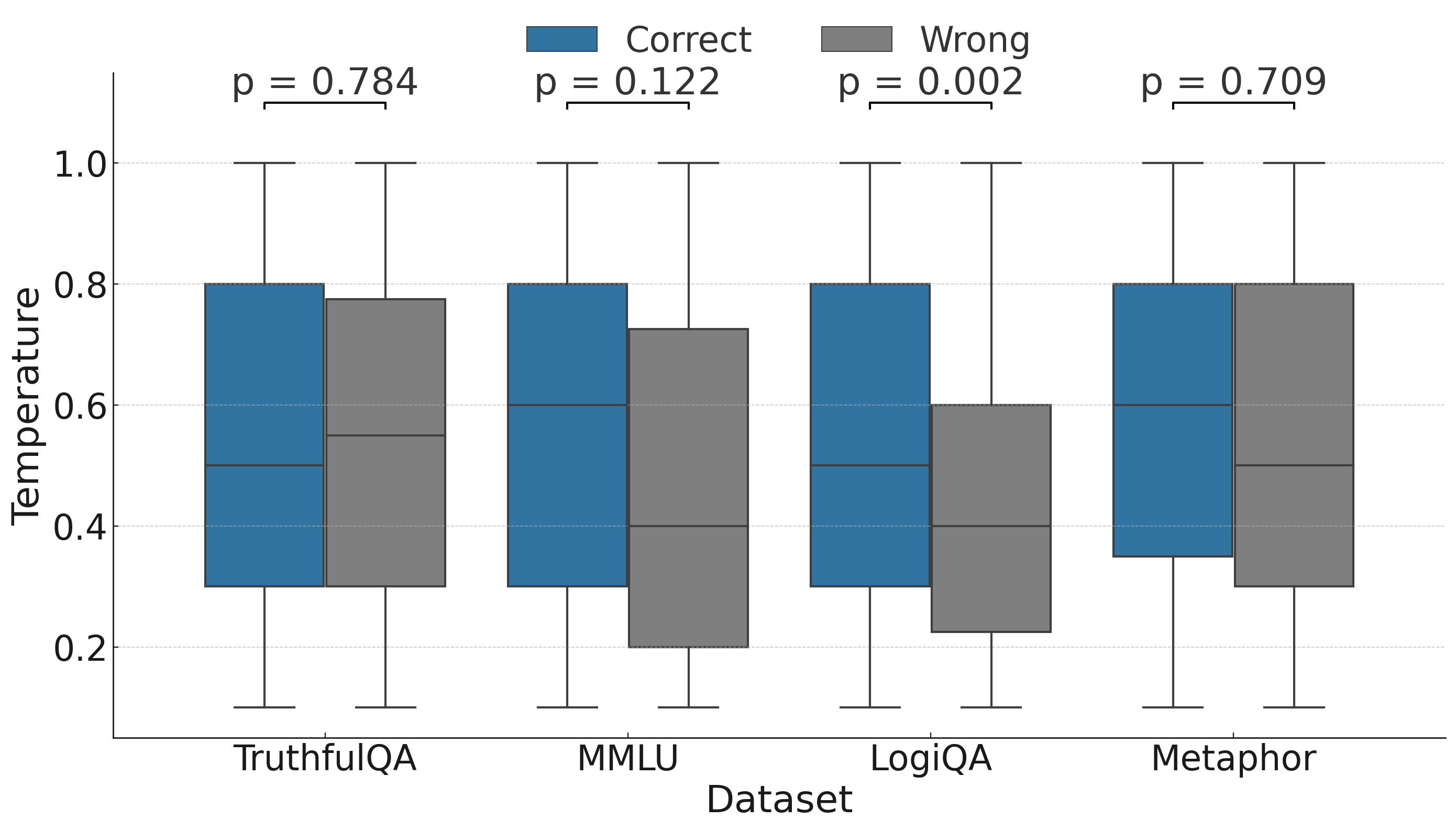}
        \caption*{(b) Temperature $a_t$}
        \label{fig:sub2}
    \end{minipage}
    \caption{Distribution of reasoning configuration action (\texttt{steps} $a_s$ and \texttt{temperature} $a_t$) across datasets, for both correctly and incorrectly answered cases.}
    \label{fig:ablation_steps_temp}
\end{figure}

\begin{table}[ht]
\centering \small
\caption{ Action Statistics across Datasets}
\begin{tabular}{lcccc}
\toprule[1pt]
\textbf{Configuration Action } & \textbf{Metaphor} & \textbf{TruthfulQA} & \textbf{MMLU} & \textbf{LogiQA}\\
\midrule
\# Steps  $a_s$       & 5.86 $\pm$ 0.57 & 6.04 $\pm$ 1.44& 6.54$\pm$ 0.71& 6.14 $\pm$ 1.02\\
 Temperature  $a_t$  & 0.542 $\pm$ 0.110 & 0.629 $\pm$ 0.281 & 0.572 $\pm$ 0.155 & 0.538 $\pm$ 0.209\\
\bottomrule[1pt]
\end{tabular}
\label{tab:avg_action_stats}
\end{table}

\autoref{fig:ablation_steps_temp} shows the boxplot  of  reasoning configuration action (\texttt{steps} $a_s$ and \texttt{temperature} $a_t$) across datasets, for both correctly and incorrectly answered cases. 
In addition, average and standard deviation statistics of $a_s$ and $a_t$ are also reported in \autoref{tab:avg_action_stats}.
While both $a_s$ and $a_p$ exhibit visibly different patterns between correct and incorrect cases across all datasets, most comparisons do not reach statistical significance. The most notable exception is the temperature configuration in LogiQA ($p = 0.002$), which shows a statistically significant gap. %This suggests that although there are apparent distributional differences, they do not consistently translate into statistically meaningful variations. 
Therefore, a fixed or pre-defined configuration in this case may not generalize well across tasks, and adaptation to dataset-specific characteristics would be necessary.

\autoref{fig:strategy-heatmap-freq-comparison} presents heatmaps of accuracy (evaluated by LLM-as-Judge) for the top-25 most frequently used $a_p$ configurations and the 25 least frequent ones, excluding strategies used only once to reduce the impact of randomness. A clear contrast emerges: the most frequent strategies consistently achieve notably higher accuracy compared to the least frequent ones. This discrepancy highlights the effectiveness of AdaReasoner in identifying and concentrating on high-performing $a_p$ instructions.

\begin{figure}[htbp]
    \centering
    \begin{minipage}{0.468\textwidth}
        \centering
        \includegraphics[width=1.2\linewidth]{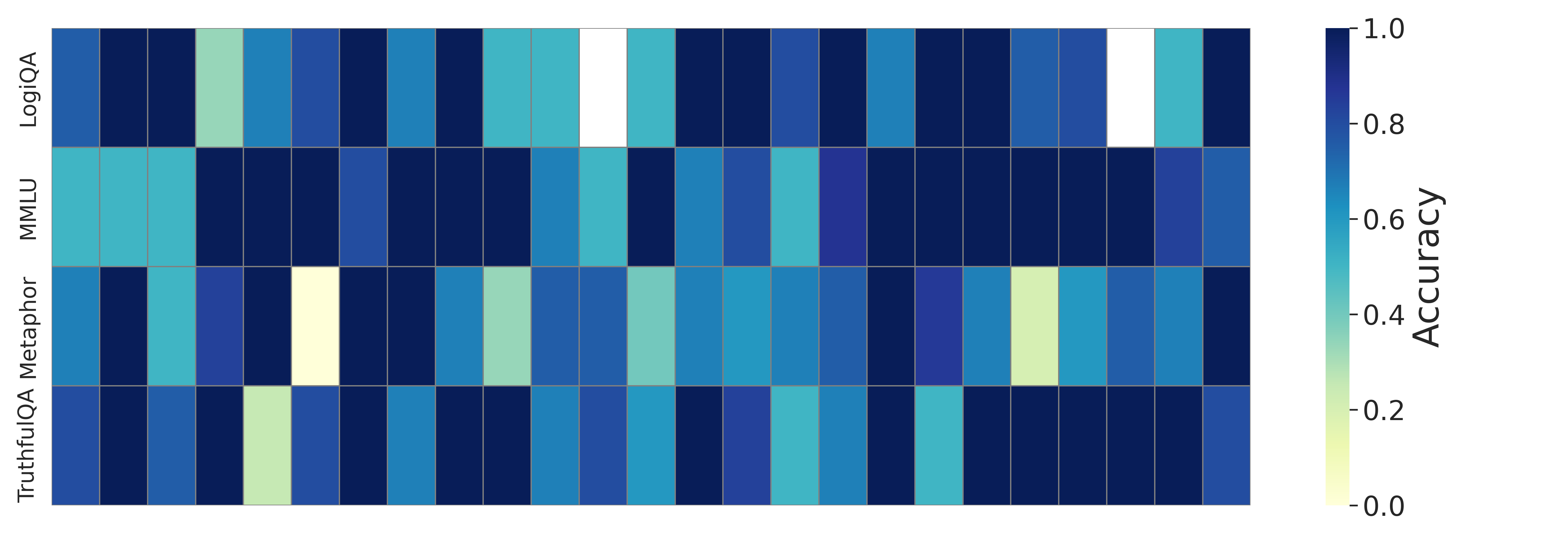}
        \caption*{(a) Top-25 most frequent $a_p$}
    \end{minipage}
    \hfill
    \begin{minipage}{0.468\textwidth}
        \centering
        \includegraphics[width=1.2\linewidth]{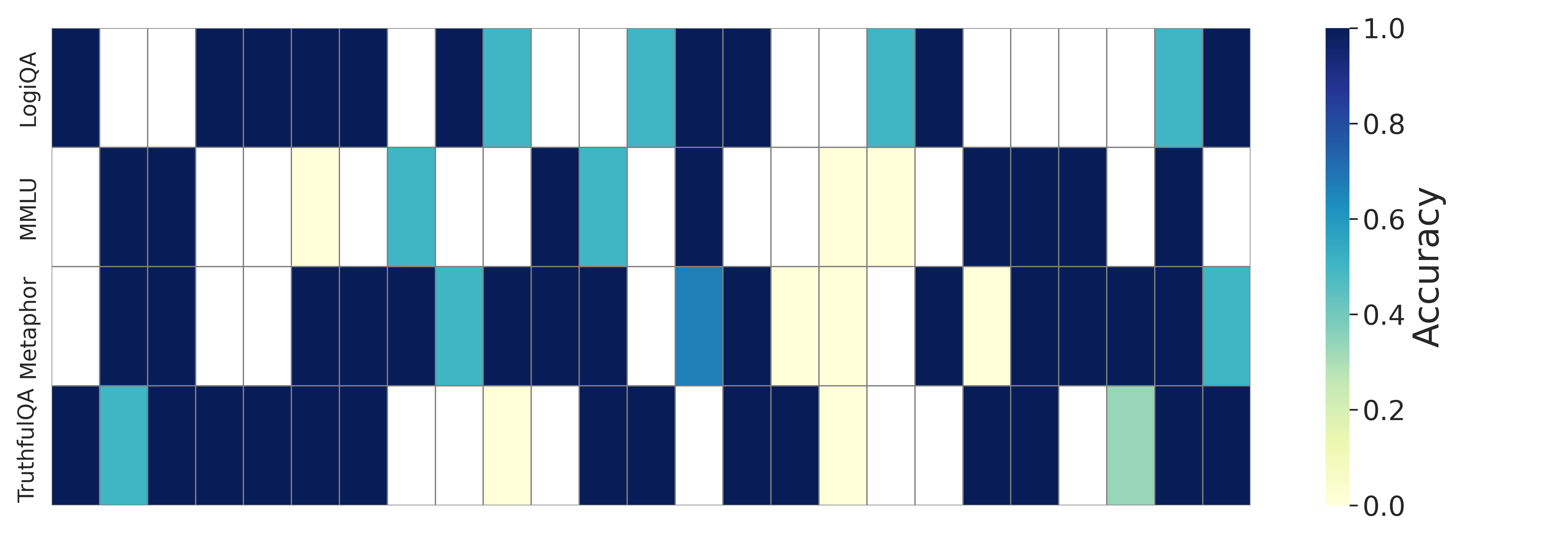}
        \caption*{(b) Top-25 least frequent $a_p$}
    \end{minipage}
    \caption{
        Comparison of $a_p$ across four datasets (LogiQA, MMLU, Metaphor, TruthfulQA).
        Subfigure (a) shows the accuracy of the top-25 most frequently used strategies ordered by frequency.
        Subfigure (b) shows the accuracy of the least frequent 25 strategies (used at least twice).
        Darker colors represent higher accuracy.
    }
    \label{fig:strategy-heatmap-freq-comparison}
\end{figure}

\autoref{tab:toptable} presents the Top-3 frequently selected  reasoning instructions $a_p$ identified by AdaReasoner  for each dataset. 
  \autoref{tab:mmluchem_top5_prompts}  shows the  Top-3 frequently selected reasoning instructions ($a_p$) identified by AdaReasoner for knowledge intensive reasoning in dataset MMLUChem.

\begin{table}[ht]
\centering \small
\caption{Top-3 reasoning instructions $a_p$ identified by AdaReasoner for each dataset}
\begin{tabular}{|c|p{12cm}|}
\hline
\textbf{Dataset} & \textbf{Action Prompt ($a_p$)} \\
\hline
\multirow{3}{*}{LogiQA} & 1. Explain your reasoning step-by-step in a \textbf{clear, accessible manner} for all audiences: Pay close \textbf{attention to minor details} that might otherwise be neglected, ensuring depth in your analysis. \\
& 2. Consider \textbf{multiple perspectives} and explore alternative viewpoints comprehensively: Decompose the problem into \textbf{smaller, logical components} to enhance clarity and precision. \\
& 3. \textbf{Reflect} on your assumptions and refine your argument through \textbf{critical self-questioning} and validation: Ensure every \textbf{premise is clearly understood} and meticulously applied. \\
\midrule
\multirow{3}{*}{MMLU} & 1. Examine the \textbf{logical connections} between concepts and articulate each step in depth: Validate your conclusions by aligning them with \textbf{established principles} or empirical data. \\
& 2. \textbf{Reflect} on your assumptions and refine your argument through \textbf{critical self-questioning} and validation: Ensure every \textbf{premise is clearly understood} and meticulously applied. \\
& 3. \textbf{Systematically} structure your analysis, elaborating on each step with thorough detail: Review and \textbf{double-check each reasoning step} to ensure no key detail is overlooked. \\
\midrule
\multirow{3}{*}{Metaphor} & 1. Include a systematic \textbf{self-check} and verification of your reasoning process to ensure consistency: Ensure every \textbf{premise is clearly understood} and meticulously applied. \\
& 2. Apply \textbf{creative reasoning} to unearth unconventional insights and challenge standard assumptions: \textbf{Challenge conventional thinking} while maintaining a sound and logical framework. \\
& 3. Consider \textbf{multiple perspectives} and explore alternative viewpoints comprehensively: \textbf{Challenge conventional thinking} while maintaining a sound and logical framework. \\
\midrule
\multirow{3}{*}{TruthfulQA} & 1. \textbf{Reflect} on your assumptions and refine your argument through \textbf{critical self-questioning and validation}: Explain your reasoning in simple, straightforward language to guarantee \textbf{clarity and accessibility}. \\
& 2. Include a \textbf{systematic self-check} and verification of your reasoning process to ensure \textbf{consistency}: Thoroughly analyze all possible interpretations to guarantee a \textbf{comprehensive understanding}. \\
& 3. Consider \textbf{multiple perspectives} and explore alternative viewpoints comprehensively: \textbf{Cross-reference} your reasoning with similar examples or prior cases for robust validation. \\
\hline
\end{tabular}
\label{tab:toptable}
\end{table}

\begin{table}[h]
\centering \small
\caption{Top-3 frequently selected reasoning instructions ($a_p$) by AdaReasoner on MMLUChem.}
\begin{tabular}{c|p{13cm}}
\toprule[1pt]
%\textbf{Top 3 Selected Prompts in MMLUChem} \\
1& Apply \textbf{creative reasoning} to unearth unconventional insights and challenge standard assumptions. \textbf{Challenge conventional thinking} while maintaining a sound and logical framework. \\
\hline
2& Conclude by summarizing your key points and \textbf{re-evaluating} your final answer for \textbf{completeness}. Thoroughly analyze \textbf{all possible interpretations} to guarantee a comprehensive understanding. \\
\hline
3& \textbf{Systematically} structure your analysis, elaborating on each step with \textbf{thorough detail}. \textbf{Cross-reference} your reasoning with similar examples or prior cases for robust validation. \\
%\hline
%Reflect on your assumptions and refine your argument through critical self-questioning and validation: Pay close attention to minor details that might otherwise be neglected, ensuring depth in your analysis. \\
%\hline
%Include a systematic self-check and verification of your reasoning process to ensure consistency: Explain your reasoning in simple, straightforward language to guarantee clarity and accessibility. \\
\bottomrule[1pt]
\end{tabular}
\label{tab:mmluchem_top5_prompts}
\end{table}

\section{Prompt Templates}

The prompt templates adopted in this study are provided in \autoref{fig:binary_template}, \autoref{fig:biy_template}, and \autoref{fig:bin_template}. \autoref{fig:binary_template} depicts the prompt format designed for binary judgment-based evaluation of LLM simulations. \autoref{fig:biy_template} shows the template applied by AdaReasoner for generating responses. \autoref{fig:bin_template} illustrates the prompts corresponding to standard CoT and the "think short" reasoning strategy.

\vspace{-10pt}
\begin{figure*}[t]
\begin{tcolorbox}[
  enhanced, % Use the enhanced version for more styling options
  colframe=brown!75!black, % Frame color
  colback=white, % Background color
  coltitle=white, % Title text color
  colbacktitle=brown!75!black, % Background color for the title
  width=\linewidth, % Box width
  arc=2mm, % Corner arc radius
  auto outer arc, % Automatically use the same arc for outer edges
  boxrule=0.5pt, % Frame rule thickness
  left=10pt, % Left padding
  right=10pt, % Right padding
  drop shadow={black!50!white},
  top=10pt, % Top padding
  bottom=10pt, % Bottom padding
  title=\textbf{Prompt Template}, % Title
  fonttitle=\bfseries, % Bold title font
  title code={\node[rounded corners, fill=blue!75!black, draw=none, text=white] at (frame.title) {\textbf{xxx}};}, % Custom title style
  attach boxed title to top center={yshift=-2mm}, % Position of the title
  boxed title style={sharp corners, size=small}, % Style of the title box
]
\small
\noindent \textbf{Assess with rigorous precision whether the provided reasoning process matches the ground truth answer.}

For a given option and response, you need to match the content of the option and response. You must not rely on the option index only, as in many cases, the index is actually incorrect.

\bigskip
\noindent \textbf{Apply these criteria for judgment and carefully consider:}

\medskip
\noindent \textbf{Mandatory Evaluation Criteria}
\begin{enumerate}
    \item \textbf{Content Equivalence}: Accept only fully equivalent numerical representations (e.g., 0.5, 50\%, 1/2) and variations in units or notation when they completely match the ground truth.
    \item \textbf{Logical Inference}: Verify that at least one reasoning step directly and logically deduces the entire correct answer in a mathematically or logically sound manner.
    \item \textbf{Substantive Matching}: For multiple-choice questions, assess the complete content of the answer (e.g., ensure "Option B" is fully equivalent to the correct answer, not just matching the label).
    \item \textbf{Semantic and Methodological Equivalence}: Recognize alternative phrasing or solution methods only if a single step unambiguously converges on the complete correct answer.
    \item \textbf{Scientific and Technical Rigor}: In technical contexts, differences in terminology, notation, or intermediate steps are acceptable only when they lead clearly and entirely to the correct conclusion.
\end{enumerate}

\noindent Using the criteria outlined above, determine whether any single rule is met--if so, the response is considered a match.

\bigskip
\noindent \textbf{Question} \\
\texttt{\{question\}}

\medskip
\noindent \textbf{Ground Truth Answer} \\
\texttt{\{correct\_answer\}}

\medskip
\noindent \textbf{Provided Reasoning} \\
\texttt{\{reasoning\_process\}}

\bigskip
\noindent Provide your final judgment as a JSON object with the following structure:

\begin{verbatim}
{
  "judge_explanation": "<brief explanation>",
  "result": "<Yes or No>"
}
\end{verbatim}

\noindent Make sure you output JSON in plain text, not as code format.

\end{tcolorbox}
\caption{Prompt template for evaluating LLM simulation by binary judgment.}
\label{fig:binary_template}
\end{figure*}

%\vspace{-30pt}
\begin{figure*}
\begin{tcolorbox}[
  enhanced, % Use the enhanced version for more styling options
  colframe=brown!75!black, % Frame color
  colback=white, % Background color
  coltitle=white, % Title text color
  colbacktitle=brown!75!black, % Background color for the title
  width=\linewidth, % Box width
  arc=2mm, % Corner arc radius
  auto outer arc, % Automatically use the same arc for outer edges
  boxrule=0.5pt, % Frame rule thickness
  left=10pt, % Left padding
  right=10pt, % Right padding
  drop shadow={black!50!white},
  top=10pt, % Top padding
  bottom=10pt, % Bottom padding
  title=\textbf{Prompt Template}, % Title
  fonttitle=\bfseries, % Bold title font
  title code={\node[rounded corners, fill=blue!75!black, draw=none, text=white] at (frame.title) {\textbf{xxx}};}, % Custom title style
  attach boxed title to top center={yshift=-2mm}, % Position of the title
  boxed title style={sharp corners, size=small}, % Style of the title box
]
\noindent\textbf{1.\ Objective}\\
Your task is to generate a \emph{comprehensive} answer to the provided question while
tailoring your reasoning and response style to the specific demands of the task.
Ensure that your answer fully adheres to the requirements \emph{without inventing any details}.

\bigskip
\noindent\textbf{2.\ Question:} \texttt{\{question\}}

\bigskip
\noindent\textbf{3.\ Adaptive Reasoning Strategy}\\
Use the following instructions to shape your response:
\texttt{\{instruction\_prompt\}}. Reason in according to the given method and adjust your reasoning approach dynamically based on the nature of the question:

\noindent You must follow \emph{no more than} \texttt{\{optimal\_steps\}} reasoning steps.

\bigskip
\noindent \textbf{Requirements:}
\begin{enumerate}
    \item Provide one answer that completely satisfies the question's requirements.
    \item Ensure your reasoning strictly adheres to the specified steps and covers all necessary details.
    \item Deliver a clear, precise, and accurate answer.
    \item Avoid repetition or ambiguity; your response should be distinct and well-reasoned.
\end{enumerate}

\end{tcolorbox}
\caption{Prompt template for AdaReasoner to generate answers.}
\label{fig:biy_template}
\end{figure*}

%\vspace{-30pt}
\begin{figure*}
\begin{tcolorbox}[
  enhanced, % Use the enhanced version for more styling options
  colframe=brown!75!black, % Frame color
  colback=white, % Background color
  coltitle=white, % Title text color
  colbacktitle=brown!75!black, % Background color for the title
  width=\linewidth, % Box width
  arc=2mm, % Corner arc radius
  auto outer arc, % Automatically use the same arc for outer edges
  boxrule=0.5pt, % Frame rule thickness
  left=10pt, % Left padding
  right=10pt, % Right padding
  drop shadow={black!50!white},
  top=10pt, % Top padding
  bottom=10pt, % Bottom padding
  title=\textbf{Prompt Template}, % Title
  fonttitle=\bfseries, % Bold title font
  title code={\node[rounded corners, fill=blue!75!black, draw=none, text=white] at (frame.title) {\textbf{xxx}};}, % Custom title style
  attach boxed title to top center={yshift=-2mm}, % Position of the title
  boxed title style={sharp corners, size=small}, % Style of the title box
]
\noindent \textbf{Please think step by step to solve the question. / Please respond fastt and think quick when solving the question. }

\bigskip
\noindent \textbf{Question:} \texttt{\{question\}}

\bigskip
\noindent \textbf{Requirements:}
\begin{enumerate}
    \item Provide one answer that completely satisfies the question's requirements.
    \item Ensure your reasoning strictly adheres to the specified steps and covers all necessary details.
    \item Deliver a clear, precise, and accurate answer.
    \item Avoid repetition or ambiguity; your response should be distinct and well-reasoned.
\end{enumerate}

\end{tcolorbox}
\caption{Prompt template for standard CoT and think short to generate answers.}
\label{fig:bin_template}
\end{figure*}

\section{Broader Impact}

AdaReasoner’s core contribution is its adaptive tuning of prompt parameters—such as instruction style, sampling temperature, and number of reasoning steps—on a per‐question basis. By automating what is traditionally a labor‐intensive trial‐and‐error process, it empowers non‐expert users to leverage large language models for diverse tasks across domains—from academic to daily commonsense—without requiring deep expertise in prompt engineering. This democratization of AI reasoning accelerates innovation and lowers barriers for users in resource‐constrained environments.

\section{LLM as Judge Reliability}
\label{sec:LLMJUDGE}
To evaluate the reliability of the LLM-as-Judge framework adopted in this study, three graduate students independently annotated three batches per dataset, each comprising 50 samples. The resulting average F1 scores (\%) across all benchmarks are reported in Table~\ref{tab:llm_judge}. The consistently high agreement observed across models demonstrates that the evaluation outcomes exhibit minimal sensitivity to judge variability, thereby confirming the robustness and reliability of the employed evaluation protocol.

\begin{table*}[h]
\centering
\caption{Average F1 scores (\%) across QA benchmarks under different reasoning strategies.}
\label{tab:llm_judge}
\resizebox{\textwidth}{!}{
\begin{tabular}{lccccccc}
\toprule
\textbf{Model} & \textbf{CoT} & \textbf{Think Short} & \textbf{ToT} & \textbf{Best-of-N} & \textbf{Auto-CoT} & \textbf{In-context CoT} & \textbf{AdaReasoner} \\
\midrule
GPT-4o & 98.83 & 99.17 & 99.17 & 99.17 & 99.50 & 99.00 & 99.00 \\
Llama-3.3-70B-Ins. & 99.50 & 100.00 & 99.17 & 99.33 & 99.00 & 98.00 & 100.00 \\
Qwen-2.5-72B-Ins. & 98.83 & 98.83 & 99.50 & 98.83 & 99.33 & 99.17 & 99.33 \\
Claude-3.5-Sonnet & 99.33 & 99.00 & 99.50 & 99.50 & 99.50 & 100.00 & 99.33 \\
DeepSeek-R1 & 99.33 & 99.17 & 99.00 & 98.83 & 99.17 & 98.00 & 100.00 \\
GPT-o3-mini & 100.00 & 100.00 & 99.00 & 100.00 & 100.00 & 99.00 & 99.50 \\
\bottomrule
\end{tabular}
}
\end{table*}

%%%%%%%%%%%%%%%%%%%%%%%%%%%%%%%%%%%%%%%%%%%%%%%%%%%%%%%%%%%%

%\appendix

%%%%%%%%%%%%%%%%%%%%%%%%%%%%%%%%%%%%%%%%%%%%%%%%%%%%%%%%%%%%

\clearpage
\section*{NeurIPS Paper Checklist}

%%% END INSTRUCTIONS %%%

\begin{enumerate}

\item {\bf Claims}
    \item[] Question: Do the main claims made in the abstract and introduction accurately reflect the paper's contributions and scope?
    \item[] Answer: \answerYes{}
    \item[] Justification: The abstract (Lines 1–15) and introduction (Section 1, Lines 16–83) accurately reflect the contributions and scope of the paper. They present AdaReasoner as an LLM-agnostic, RL-based reasoning configuration adapter with a factorized action space and Boltzmann exploration, supported by theoretical guarantees (Appendix B) and extensive empirical validation (Section 4). Key contributions—such as the few-shot convergence (Figure 3) and outperformance across six LLMs and four datasets (Table 1)—are consistently stated up front and substantiated in subsequent sections.

\item {\bf Limitations}
    \item[] Question: Does the paper discuss the limitations of the work performed by the authors?
    \item[] Answer: \answerYes{} % Replace by \answerYes{}, \answerNo{}, or \answerNA{}.
    \item[] Justification:  Section 6 (Lines 363–367) outlines key limitations: reliance on few-shot tuning, poor cross-model transferability, and RL-induced computational overhead. These are also supported by empirical evidence in Section 4.3.

\item {\bf Theory assumptions and proofs}
    \item[] Question: For each theoretical result, does the paper provide the full set of assumptions and a complete (and correct) proof?
    \item[] Answer: \answerYes{} % Replace by \answerYes{}, \answerNo{}, or \answerNA{}.
    \item[] Justification: The full assumptions—e.g., smoothness and bounded gradient variance—are clearly stated, and complete proofs are provided in Appendix B, including detailed convergence theorem and regret analysis.

    \item {\bf Experimental result reproducibility}
    \item[] Question: Does the paper fully disclose all the information needed to reproduce the main experimental results of the paper to the extent that it affects the main claims and/or conclusions of the paper (regardless of whether the code and data are provided or not)?
    \item[] Answer: \answerYes{}
    \item[] Justification: Section 4.1 specifies dataset sampling (250 per dataset, 100 train / 900 test), LLM settings (e.g., top-p=0.1, max tokens=5000), few-shot selections, and baseline configurations. Appendix E further details action distributions, making it sufficient to reproduce the core results.

\item {\bf Open access to data and code}
    \item[] Question: Does the paper provide open access to the data and code, with sufficient instructions to faithfully reproduce the main experimental results, as described in supplemental material?
    \item[] Answer: \answerYes{} % Replace by \answerYes{}, \answerNo{}, or \answerNA{}.
    \item[] Justification: Data code is provided via annonymous github repository link: https://anonymous.4open.science/r/officialadareasoner-B9B

\item {\bf Experimental setting/details}
    \item[] Question: Does the paper specify all the training and test details (e.g., data splits, hyperparameters, how they were chosen, type of optimizer, etc.) necessary to understand the results?
    \item[] Answer: \answerYes{}
    \item[] Justification: Section 4.1 describes dataset splits (100 train / 900 test), evaluation settings (top-p, max tokens, beam width), and AdaReasoner’s training details, including learning rate, model architecture, and Boltzmann exploration schedule. Discretized action spaces are defined in Section 3.1 and Appendix C.

\item {\bf Experiment statistical significance}
    \item[] Question: Does the paper report error bars suitably and correctly defined or other appropriate information about the statistical significance of the experiments?
    \item[] Answer: \answerYes{} % Replace by \answerYes{}, \answerNo{}, or \answerNA{}.
    \item[] Justification: The paper reports standard deviations for key variables (e.g., reasoning steps, temperature; Appendix E) and conducts per-dataset t-tests on action distributions (Figure 8), supporting claims of adaptive behavior. Main accuracy results are single-run due to API cost.

\item {\bf Experiments compute resources}
    \item[] Question: For each experiment, does the paper provide sufficient information on the computer resources (type of compute workers, memory, time of execution) needed to reproduce the experiments?
    \item[] Answer: \answerNo{} % Replace by \answerYes{}, \answerNo{}, or \answerNA{}.
    \item[] Justification: AdaReasoner’s few-shot RL fine-tuning is lightweight and should be adapted on any modern computational resource as per described in model parameter settings.
    
\item {\bf Code Of ethics}
    \item[] Question: Does the research conducted in the paper conform, in every respect, with the NeurIPS Code of Ethics \url{https://neurips.cc/public/EthicsGuidelines}?
    \item[] Answer: \answerYes{}
    \item[] Justification: The work develops an algorithmic framework and evaluates it on public benchmarks without using sensitive or proprietary data, does not involve human or animal subjects, and poses no foreseeable misuse beyond standard LLM research.

\item {\bf Broader impacts}
    \item[] Question: Does the paper discuss both potential positive societal impacts and negative societal impacts of the work performed?
    \item[] Answer: \answerYes{}
    \item[] Justification: Discussion of broader impacts have been discussed in Appendix G. 
    
\item {\bf Safeguards}
    \item[] Question: Does the paper describe safeguards that have been put in place for responsible release of data or models that have a high risk for misuse (e.g., pretrained language models, image generators, or scraped datasets)?
    \item[] Answer: \answerNA{}
    \item[] Justification: AdaReasoner is a lightweight RL-based adapter evaluated on publicly available benchmarks and does not involve releasing any new pretrained models or scraped datasets that would pose dual-use or safety risks, so no additional safeguards are necessary.

\item {\bf Licenses for existing assets}
    \item[] Question: Are the creators or original owners of assets (e.g., code, data, models), used in the paper, properly credited and are the license and terms of use explicitly mentioned and properly respected?
    \item[] Answer: \answerYes{}
 % Replace by \answerYes{}, \answerNo{}, or \answerNA{}.
    \item[] Justification: Used datasets are cited with their proper paper or website links.

\item {\bf New assets}
    \item[] Question: Are new assets introduced in the paper well documented and is the documentation provided alongside the assets?
    \item[] Answer: \answerNA{}
    \item[] Justification:  The paper does not introduce or release any new datasets, codebases, or models;

\item {\bf Crowdsourcing and research with human subjects}
    \item[] Question: For crowdsourcing experiments and research with human subjects, does the paper include the full text of instructions given to participants and screenshots, if applicable, as well as details about compensation (if any)? 
    \item[] Answer: \answerNA{}
    \item[] Justification: The paper does not involve any crowdsourcing experiments or research with human subjects.

\item {\bf Institutional review board (IRB) approvals or equivalent for research with human subjects}
    \item[] Question: Does the paper describe potential risks incurred by study participants, whether such risks were disclosed to the subjects, and whether Institutional Review Board (IRB) approvals (or an equivalent approval/review based on the requirements of your country or institution) were obtained?
    \item[] Answer: \answerNA{}
    \item[] Justification: No human subjects were involved in this research.

\item {\bf Declaration of LLM usage}
    \item[] Question: Does the paper describe the usage of LLMs if it is an important, original, or non-standard component of the core methods in this research? Note that if the LLM is used only for writing, editing, or formatting purposes and does not impact the core methodology, scientific rigorousness, or originality of the research, declaration is not required.
    %this research? 
    \item[] Answer: \answerYes{} % Replace by \answerYes{}, \answerNo{}, or \answerNA{}.
    \item[] Justification: The paper proposes AdaReasoner as a plugin to adaptively configure LLM reasoning parameters. LLMs serve as the target models whose responses are evaluated and optimized through reinforcement learning. The usage is central to the methodology and fully described in Sections 1 and 3.

\end{enumerate}

\end{document}